\documentclass{article}

\usepackage[preprint]{neurips_2021}
\usepackage{microtype}
\usepackage{graphicx}
\usepackage{mathtools}
\usepackage{amssymb, multirow, paralist, color,amsmath,amsthm}
\usepackage{subfigure}
\usepackage{booktabs} 
\usepackage{algorithm,algorithmic}
\usepackage{amssymb, multirow, paralist, color}
\newtheorem{thm}{Theorem}

\newtheorem{lemma}{Lemma}
\newtheorem{cor}{Corollary}

\newtheorem{ass}{Assumption}
\usepackage{enumitem}
\usepackage{textcase}
\usepackage[T1]{fontenc}
\usepackage{makecell}
\usepackage{fancybox}
\usepackage{listings}

\usepackage{hyperref}

\def \S {\mathbf{S}}
\def \A {\mathcal{A}}
\def \B {\mathcal{B}}

\def \R {\mathbb{R}}

\def \w {\mathbf{w}}
\def \v {\mathbf{v}}
\def \m {\mathbf{m}}
\def \t {\mathbf{t}}
\def \x {\mathbf{x}}

\def \E {\mathbb{E}}

\def \x {\mathbf{x}}

\def \a {\mathbf{a}}

\def \b {\mathbf{b}}

\def \1 {\mathbf{1}}

\newcommand{\Norm}[1]{\left\|#1\right\|}

\def \z {\mathbf{z}}

\def \u {\mathbf{u}}

\def \P {\mathcal{P}}

\def \bg {\mathbf{g}}

\def \B {\mathcalB}

\def \x {\mathbf{x}}

\def \D {\mathcal{D}}
\def \z {\mathbf{z}}
\def \u {\mathbf{u}}

\def \w {\mathbf{w}}

\def \R {\mathbb{R}}
\def \S {\mathcal{S}}

\def \A {\mathcal{A}}

\def \v {\mathbf{v}}

\def \a {\mathbf{a}}
\def \b {\mathbf{b}}

\def \B {\mathcal{B}}

\def \t {\mathbf{t}}

\def\inner#1#2{\left\langle #1, #2 \right\rangle}

\newcommand\myfontsize{\fontsize{16pt}{18pt}\selectfont}

\definecolor{commentcolor}{RGB}{110,154,155}   
\newlength\myindent
\definecolor{codegreen}{rgb}{0.313, 0.498, 0.498} 
\definecolor{codegray}{rgb}{0.5,0.5,0.5}
\definecolor{codepurple}{rgb}{0.58,0,0.82}
\definecolor{backcolour}{rgb}{0.95,0.95,0.92}

\lstdefinestyle{mystyle}{
    commentstyle=\color{codegreen},
    keywordstyle=\color{magenta},
    numberstyle=\tiny\color{codegray},
    stringstyle=\color{codepurple},
    basicstyle=\ttfamily\footnotesize,
    breakatwhitespace=false,         
    breaklines=true,                 
    captionpos=b,                    
    keepspaces=true,                 
    numbers=left,                    
    numbersep=5pt,                  
    showspaces=false,                
    showstringspaces=false,
    showtabs=false,                  
    tabsize=2
}

\begin{document}
\title{\myfontsize{Provable Stochastic Optimization for Global Contrastive Learning: Small Batch Does Not Harm Performance}}

\author{%
 Zhuoning Yuan$^1$, Yuexin Wu$^2$, Zi-Hao Qiu$^3$, Xianzhi Du$^2$, \\ \textbf{Lijun Zhang$^3$, Denny Zhou$^2$, Tianbao Yang$^1$\thanks{Correspondence to zhuoning-yuan@uiowa.edu,  dennyzhou@google.com, tianbao-yang@uiowa.edu. The code for experiments is available at \url{https://github.com/Optimization-AI/sogclr}.}} \\
  $^1$Department of Computer Science, the University of Iowa \\
  $^2$Google Research \\
  $^3$National Key Laboratory for Novel Software Technology, Nanjing University \\
}

\maketitle

\begin{abstract}
In this paper, we study contrastive learning from an optimization perspective, aiming to analyze and address a fundamental issue of existing contrastive learning methods that either rely on a large batch size or a large dictionary of feature vectors. We consider a global objective for contrastive learning, which contrasts each positive pair with all negative pairs for an anchor point. From the optimization perspective, we explain why existing methods such as SimCLR require a large batch size in order to achieve a satisfactory result. In order to remove such requirement, we propose a memory-efficient \textbf{S}tochastic \textbf{O}ptimization algorithm  for solving the \textbf{G}lobal objective of \textbf{C}ontrastive \textbf{L}earning of \textbf{R}epresentations, named SogCLR. We show that its optimization error is negligible  under a reasonable condition  after a sufficient number of iterations  or is diminishing for a slightly different global contrastive objective. Empirically, we demonstrate that SogCLR with small batch size (e.g., 256) can achieve similar performance as SimCLR with large batch size (e.g., 8192) on self-supervised learning task on ImageNet-1K. We also attempt to show that the proposed optimization technique is generic and can be applied to solving other contrastive losses, e.g., two-way contrastive losses for bimodal contrastive learning. The proposed method is implemented in our open-sourced library LibAUC (\url{www.libauc.org}).
\end{abstract}

\section{Introduction}
\label{introduction}
Recently, self-supervised learning (SSL) for pre-training deep neural networks, which springs from natural language processing~\cite{mikolov2013efficient,devlin2018bert,lan2019albert},  has emerged to be a popular  paradigm in computer vision for learning visual representations~\cite{dosovitskiy2020image,zhu2020deformable,liu2021swin}. A simple yet effective framework of SSL for learning visual representations is contrastive learning~\cite{chopra2005learning, simclrv1}, which uses the gradient of a contrastive loss to update model, aiming to push the similarity scores between positive pairs (augmented data from the same image) to be higher than that between negative pairs (augmented data from different images). 

While the great performance of contrastive learning methods and their alternatives have been demonstrated on popular benchmarks (e.g., ImageNet), some fundamental problems of contrastive learning remain unresolved. One such problem is the requirement for large batch size. Unlike supervised learning methods, the performance of
SimCLR~\cite{simclrv1} decreases as the batch size decreases, and a satisfactory performance can be only achieved with a large batch size on natural image  datasets (e.g., 8192 for ImageNet). However, in practice, training models with such a large batch size can be memory-intensive and requires more computational resources, especially when adopting large-scale backbones (e.g., Vision Transformers~\cite{Dosovitskiy2021AnII, Zhai2021ScalingVT}) or taking video sequences as input~\cite{Qian2021SpatiotemporalCV}.

To address this issue, some ad-hoc approaches have been investigated. For example, the MoCo method~\cite{mocov1} uses a large dictionary to maintain a set of feature vectors for constructing negative pairs with data in the mini-batch. Other approaches choose to get around such issue by optimizing pairwise loss~\cite{byol,pmlr-v139-zbontar21a,DBLP:conf/cvpr/ChenH21} or other losses~\cite{swav}. Nevertheless, the fundamental issue of optimizing a contrastive loss with a large batch size requirement still exists. This also occurs in other tasks with a similar contrastive loss, e.g., bimodal SSL tasks by optimizing a two-way contrastive loss (e.g., CLIP~\cite{clip}).

In this paper, we aim to address this fundamental problem from the optimization perspective by considering a global objective for contrastive learning, providing a rigorous analysis to explain why SimCLR requires a large mini-batch size, and designing a memory-efficient stochastic algorithm for optimizing the global contrastive objective with provable convergence guarantee under a reasonable condition.  Our major contributions are summarized below: 
\begin{itemize}
      \item We propose a global objective for contrastive learning, in which the similarity score between a random positive pair of an anchor point is contrasted with that between the anchor point and all other images and their augmented data. We cast the problem as a special case of  coupled compositional stochastic optimization  by highlighting the challenges in designing stochastic algorithms. 
    \item We analyze SimCLR from the perspective of optimizing the global contrastive objective, and show that it suffers from 
    an optimization error of SimCLR in the order of $O(1/\sqrt{B})$ for the objective's gradient norm even with the number of iterations approaching infinity, which explains the phenomena that SimCLR's performance degrades as the mini-batch size decreases.
    \item We propose a memory-efficient stochastic algorithm named SogCLR without relying on a large batch size. We establish the convergence of the proposed algorithm SogCLR and show that its optimization error for the aforementioned global contrastive objective is negligible under a mild condition. Moreover, we show that SogCLR converges to a stationary solution to a slightly different global contrastive objective with a diminishing optimization error as the number of iterations increases. 
    \item We demonstrate the empirical success of SogCLR on ImageNet-1K. With a standard mini-batch size 256 and the same other settings as SimCLR, by running 800 epochs, SogCLR achieves a performance of 69.4\% for top 1  linear evaluation accuracy, which is better than 69.3\% of SimCLR using a large batch size 8,192. The comparison between SimCLR and SogCLR by varying different batch sizes is shown in Figure~\ref{fig:impact_of_batch_size}. 
    
    \item We further incorporate other useful techniques into SogCLR, e.g., multi-crop augmentation and multiple MLP projection heads, and we are able to achieve 72.5\% top-1 linear evaluation accuracy on ImageNet-1K, which is competitive with existing listwise contrastive loss based SSL methods using a large dictionary (e.g., MoCo-v2). We further demonstrate the usefulness of the proposed technique for bimodal contrastive learning, e.g., CLIP. 
\end{itemize}
Finally, we would like to emphasize that to the best of our knowledge, this is the first work that analyzes SimCLR and a stochastic algorithm for contrastive learning from an optimization perspective. We expect that this paper would inspire new  studies by proposing better algorithms for optimizing the global contrastive objective. 
\section{Related Work}
We would like to point out that SSL is an emerging field and there are tremendous studies proposing different methodologies.  Nevertheless, we focus our attention on different methodologies for contrastive SSL.  

{\bf Contrastive Losses.} There are multiple definitions of contrastive loss, including pairwise losses, and listwise losses. The notation of contrastive loss dates back to 15 years ago for dimensionality reduction~\cite{10.1109/CVPR.2006.100}, which uses pairwise contrastive losses that simply push the similarity scores between positive pairs to be high and that between negative pairs to be low. Listwise contrastive losses have been proposed in the context of distance metric learning~\cite{sohn2016improved}, which contrasts a similarity score between an anchor point and a positive sample with  a number of similarity scores between the anchor point and multiple  negative samples. \cite{oord2018representation} is a pioneering work that uses a contrastive loss for unsupervised representation learning. They propose a contrastive loss based on noise contrastive estimation (NCE), which is called InfoNCE. It was used to learn representations by predicting the future in the latent space by using autoregressive model. However, a fundamental issue regarding how to select the negative samples and how it affects the learning performance was not studied in~\cite{oord2018representation}. 

\textbf{Contrastive SSL.} The InfoNCE loss  was later adopted in the momentum contrast (MoCo) method~\cite{mocov1} for SSL of visual representations. MoCo tackles the question of how to construct negative samples in the latent space.  It introduces two techniques (i) a momentum encoder network, which is used to generate representations of images for contrast with that generated by the target network on the anchor points, and is updated by a momentum step; (ii)  a large dictionary that stores a number of feature representations for constructing negative pairs that are generated by the momentum encoder network, and is updated by a queue structure in a FIFO fashion.  Later, the large dictionary and momentum contrast was abandoned in SimCLR~\cite{simclrv1}, which uses a large batch to sample data for constructing positive and negative pairs within the batch. SimCLR makes several contributions for improving the performance, in particular using strong data augmentations and MLP projection layers. SimCLR conducted extensive experiments by studying how the batch size and other factors (e.g. number of epochs) affect the performance and a key observation is that the performance degrades as the mini-batch size decreases.  Although a large-batch size is preferred or not an issue in industrial setting, the fundamental issue of requiring large batch size is still not well addressed. 

\begin{figure}[t]
\centering
\includegraphics[scale=0.5]{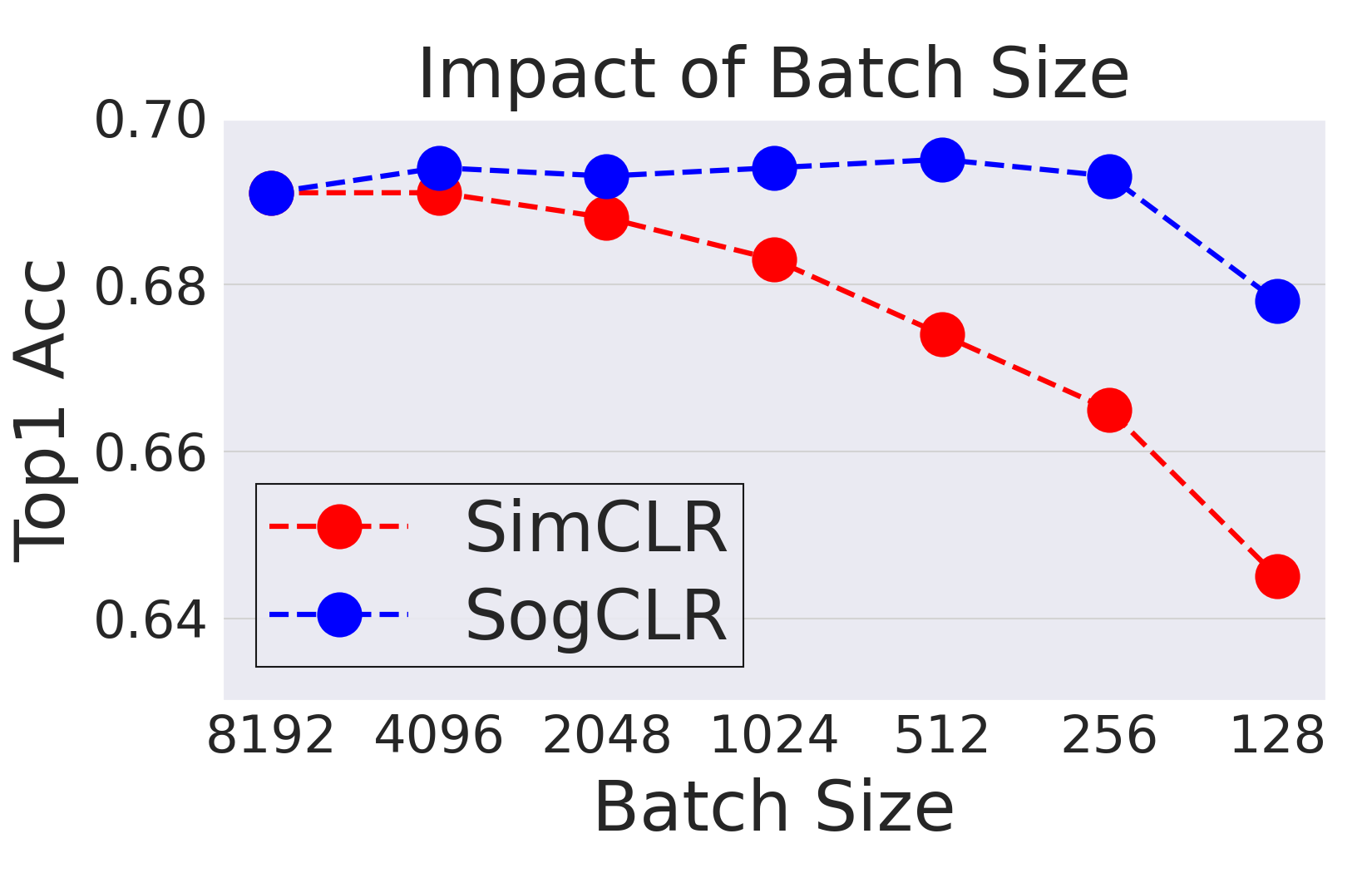}
\vspace{-0.1in}
\caption{Impact of batch size. X-axis is batch size. Y-axis is linear evaluation accuracy with 800-epoch pretraining with ResNet-50 on ImageNet-1K.}
\label{fig:impact_of_batch_size} 
\vspace{-0.3in}
\end{figure} 


\textbf{Improvements on top of MoCo or SimCLR}. Recently, there have been some efforts made for improving MoCo and SimCLR. MoCo-v2~\cite{mocov2} is an improved version of MoCo by adopting strong data augmentations and multiple MLP projection heads as in SimCLR. MoCo-v2 (with ResNet-50 as backbone) achieves 71.1\% top-1 linear evaluation accuracy on ImageNet-1K with small batch size of 256, which suppresses the SimCLR’s 69.3\% with larger batch size of 8192. 

Several works have tried to approach SSL from the InfoMax principle, i.e., maximizing mutual information between different views and the shifted negative InfoNCE loss is a lower bound of  mutual information~\cite{tian2020makes,chen2021simpler,zhu2020eqco}.  \cite{tian2020makes} studies the effects of data augmentations from the perspective of mutual information between different views, and uses more and stronger data augmentations  on the top of MoCo-v2 framework, which achieves 73\% top-1 linear evaluation accuracy on Imagenet-1K with ResNet-50. \cite{chen2021simpler} clarifies why InfoNCE loss fails under small batch-size settings by showing that the negative InfoNCE loss approaches saturation (e.g., $\log(B)$ where $B$ is batch size) after a few epochs.  To address this issue, they propose a self-normalized version of InfoNCE loss named FlatNCE, in which larger weights will prioritize harder negative samplers in a mini-batch to facilitate the models learning better representations. \cite{zhu2020eqco} proposes a similar technique from the viewpoint of mutual information lower bound to address the small batch collapse issue. In detail, they add a constant margin to offset the similarity score between positive pairs, which could help increase the weight of learning from hard negative samples.

\cite{li2020prototypical,dwibedi2021little, chuang2020debiased} improve InfoNCE in a different way. For example, \cite{li2020prototypical} proposes ProtoNCE loss that uses prototypical representations in place of the second view in InfoNCE loss, which are learned by clustering the data into a large number of clusters. \cite{dwibedi2021little} uses similar samples for a given image extracted from a support set maintained by a queue to improve the data diversity for training. \cite{chuang2020debiased} proposes a debiased contrastive loss to tackle the sampling bias (false negative samples) and they observe sampling negative samples from the truly different labels improves the performance. We compare our method with several InfoNCE-loss based SSL methods in Table~\ref{tab:ssl_all} with top-1 linear evaluation accuracy reported on ImageNet-1K. 

{\bf {Global Contrastive Loss}.} The global contrastive loss that contrasts a positive pair with all possible negative pairs has been explored in the literature. For example, \cite{wang2020understanding} formulates the InfoNCE loss with the number of negative samples approaching infinity into two parts and explain them from the perspectives of alignment and uniformity, where alignment aims to keep similar positive pairs closer and uniformity aims to perverse the maximal information among all pairs by pushing them evenly distributed on the hypersphere. 
However, The challenge of handling a large number of components in the normalization term in the InfoNCE loss is not well addressed. One of the most well-known techniques to address this challenge is to use NCE~\cite{gutmann2010noise}, which reduces the problem to how to sample negative data and how many negative data are sufficient for obtaining satisfactory performance. Although the original paper of NCE shows that the approximation error of NCE decreases as the number of samples increases, it is unclear how it affects the performance of SSL for visual representations. In contrast, this paper provides an arguable better approach for tackling the global contrastive loss in the sense that (i) the performance does not hinge on how to sample negative samples and how many to sample; (ii) there is a stronger convergence guarantee for optimizing the global contrastive loss. 

\textbf{Evolution of State-of-the-art.} \cite{swav} proposes SwAV that solves a swapped prediction problem wherein the prototypical codes obtained from one data augmented view are predicted using the similarity scores between the other view's representations and the prototypical codes.  They also propose a multi-crop augmentation strategy, which significantly boosts the performance under small batch setting. In particular, SwAV achieves a top-1 linear evaluation accuracy of 75.3\% using a batch size of 256 on ImageNet-1K. \cite{byol} proposes a method named BYOL which minimizes a pairwise loss based on feature representations of two data augmented views from two neural networks, referred to as online and target
networks, respectively.  From an augmented view of an image, an online network is learned to predict the target network representation of the same image under a different
augmented view. The target network is updated by a momentum step as in MoCo. BYOL achieves  74.3\% top-1  linear evaluation  accuracy on ImageNet-1K with a ResNet-50.  Recently, a concurrent work \cite{ReLiCv2} attains state-of-the-art performance of top-1 linear evaluation accuracy (77.1\%) by using a combination of techniques, e.g., the InfoNCE loss as in SimCLR/MoCo, the online-target setup as in MoCo/BYOL, the multi-crop augmentations as in SwAV, and the invariance loss as in~\cite{ReLiCv1}. We also notice that larger backbones lead to better performance, e.g, ResNet-200~\cite{byol,ReLiCv2} and Vision Transformer~\cite{mocov3,li2021efficient}. We would like to point out that it is not our focus to leverage all of these different techniques to achieve state-of-the-art performance. But instead we focus on understanding the fundamental limits of optimizing the InfoNCE loss and providing an alternative yet effective strategy to make contrastive learning possible without using a large batch size, which is potentially useful for different methods, e.g., supervised contrastive learning~\cite{khosla2020supervised}, and bimodal contrastive learning mentioned below~\cite{clip}. 

Recently, a two-way contrastive loss has been used in bimodal contrastive learning~\cite{zhang2021contrastive,clip}, which takes paired image and text as input and aims to map them into a Euclidean space that are closer than non-observed image/text pairs.  The CLIP model~\cite{clip} uses this idea to train a model on a large image-text dataset and observe promising zero-shot prediction performance for downstream tasks. However, their approach also uses a very large batch size equal to 32,768.

\section{Optimizing Global Contrastive Objective}
{\bf Notations.} Let $\D=\{\x_1, \ldots, \x_n\}$ denote the set of training images, let $\mathcal P$ denote a set of data augmentation operators that can be applied to each image to generate a copy. Let $\A(\cdot)\in\mathcal P$ denote a random data augmentation operator, and let  $\x\in\D$ denote a random example from $\D$. 
Let $\S_i=\{\A(\x): \forall\A\in\P, \forall\x\in\D\setminus \{\x_i\}\}$ denote all training images including their augmented versions but excluding that of $\x_i$. Let $E(\cdot)$ denote the encoder network parameterized by $\w\in\R^d$ that outputs a normalized feature representation of an input image.  Below, $\A(\x_i)$ and $\A(\x_j)$ denote two independent random data augmentations applied to $\x_i$ and $\x_j$ independently. 

The SimCLR method is to update the model according to the gradient of the local contrastive loss that is defined over sampled mini-batch data. To this end, a random mini-batch of $B$ images $\B=\{\x_1, \cdots, \x_B\}$ are first sampled.  Then for each image $\x_i\in\B$, two random augmented data $\A(\x_i), \A'(\x_i)$ are generated by two randomly sampled data augmentations $\A, \A'\in\P$. Then the gradient is computed based on the following local contrastive loss for each data $\x_i$ and its symmetric one by switching $\A$ and $\A'$: 
\begin{align}\label{eqn:con0} 
L_{\B}(\w; \x_i, \A, \A') = -\ln\frac{\exp(E(\A(\x_i))^{\top}E(\A'(\x_i))/\tau)}{g(\w; \x_i, \A, \B)}, 
\end{align}
where $\tau$ is known as the temperature parameter, and 
\begin{align}\label{eqn:sampleg}
  g(\w; \x_i, \A, \B_i) & = \sum_{\z_j\in\B_i}(\exp(E(\A(\x_i))^{\top}E(\z_j)/\tau)  
\end{align}
and $\B_i = \{\A(\x_j), \A'(\x_j): \x_j\in \B\setminus\{\x_i\}\}$ denote the set of images that are generated by applying independent two random data augmentations to each image in $\B$ independently  excluding $\x_i$. 

\subsection{A Global Contrastive Objective: V1}
The local contrastive loss defined over the mini-batch samples hides the complexity for contrastive learning, which renders the SimCLR method sensitive to the mini-batch size. To address this issue, we propose a global contrastive objective.  To this end, we define the following global contrastive loss for each augmented data pair $(\A(\x_i), \A'(\x_i))$:
\begin{align}\label{eqn:con}
L(\w; \x_i, \A, \A') = -\ln\frac{\exp(E(\A(\x_i))^{\top}E(\A'(\x_i))/\tau)}{\varepsilon' + g(\w; \x_i, \A, \S_i)},
\end{align}
where $\varepsilon'>0$ is a small constant, which is introduced simply for the purpose of analysis to ensure the denominator that involves $g(\w; \x_i, \A, \S_i)$ is lower bounded~\footnote{We can also modify the definition of $\S_i$ to include $\A(\x_i)$ for ensuring $g(\w; \x_i, \A, \S_i)$ is lower bounded without adding $\varepsilon'$.}, and 
\begin{align}\label{eqn:g}
  g(\w; \x_i, \A, \S_i) & = \sum_{\z\in\S_i}(\exp(E(\A(\x_i))^{\top}E(\z)/\tau),
\end{align}
which contrasts the similarity score between each positive pair $E(\A(\x_i))^{\top}E(\A'(\x_i)$ with the similarity scores of negative pairs $E(\A(\x_i))^{\top}E(\z)$ for all $\z\in\S_i$. Based on the individual contrastive loss, we define the following {\bf global contrastive objective (GCO)} for minimization: 
\begin{align}\label{eqn:gcl}
   \min_{\w} F(\w) = \E_{\x_i\sim\D, \A, \A'\sim \P} [\tau L(\w; \x_i, \A, \A')]
\end{align}
where $\sim$ denotes a random sample, $L$ is multiplied by $\tau$ to ensure the gradient is not illy scaled. In contrast to another variant proposed in section~\ref{sec:4}, we refer to the above objective as the V1 GCO. 

To highlight the challenge for optimizing the global contrastive objective, we consider the calculation of the gradient of $\tau L(\w; \x_i, \A, \A')$ in terms of the parameters $\w$ of the encoder network $E$. 
\begin{align*}
    &\tau\nabla L(\w; \x_i, \A, \A')  =  - \nabla (E(\A(\x_i))^{\top}E(\A'(\x_i)))\\
    & \quad\quad+ \frac{\tau}{\varepsilon'  + g(\w; \x_i, \A, \S_i)}\nabla g(\w; \x_i, \A,  \S_i).
\end{align*}
It is notable that the first term can be easily computed by back-propogation. The challenge lies at computing the second term, where $g(\w; \x_i, \A, \S_i)$ involves a large number of examples in $\S_i$ that includes all images and their augmented data excluding that of $\x_i$. Due to the finite-sum structure of $g$ in~(\ref{eqn:g}), we can compute an unbiased estimator by sampling data from $\S_i$. Indeed, we can show that $\frac{1}{|\B_i|}g(\w; \x_i, \A,  \B_i)$ is an unbiased estimator of $\frac{1}{|\S_i|}g(\w; \x_i, \A, \S_i)$. SimCLR directly uses this mini-batch estimator to estimate $g(\w; \x_i, \A, S_i)$ and $\nabla g(\w; \x_i, \A, \S_i)$ in the above equation, yielding the following approximated gradient of $\tau L(\w; \x_i, \A, \A')$:
\begin{align}\label{eqn:g2}
    \tau\widehat\nabla L(\x_i, \A, \A')  &=  - \nabla (E(\A(\x_i))^{\top}E(\A'(\x_i)))\\
    & + \frac{\tau}{\varepsilon + g(\w; \x_i, \A,  \B_i)}\nabla g(\w; \x_i, \A,  \B_i),\notag
\end{align}
where $\varepsilon=\frac{|\B_i|\varepsilon'}{|\S_i|}$.  However, this quantity is a biased estimator of $\tau\nabla L(\w; \x_i, \A, \A')$ due to the non-linear function $ \frac{1}{\varepsilon+g(\w; \x_i, \A,  \B_i)}$.  

\subsection{SimCLR and its Convergence for V1 GCO}
The SimCLR method can be viewed as a mini-batch based stochastic method, which uses a gradient estimator that is the average of the estimator in~(\ref{eqn:g2}) for $\x_i$ in the sampled mini-batch. To analyze the optimization error of SimCLR, we first consider the following simplest update~\footnote{For simplicity, we do not include another similar term in the gradient estimator by switching $\A$ and $\A'$, which will not affect the analysis.}: 
\begin{align}\label{eqn:SimCLR}
    \w_{t+1} = \w_t - \eta\frac{1}{B}\sum_{\x_i\in\B}\widehat\nabla L(\x_i, \A, \A').
\end{align}
We establish the optimization error of the above update for $T$ iterations for optimizing the V1 GCO.
\begin{thm}\label{thm:1}
Assume $F$ is smooth, $g$ is smooth and Lipchitiz continuous, SimCLR with the update~(\ref{eqn:SimCLR}) ensures that $\E[\|\nabla F(\w_{t'})\|^2]\leq O(\frac{1}{\eta T} + \eta + \frac{1}{B})$ for a random $t'\in\{1,\ldots, T\}$.  
\end{thm}
{\bf Remark:} The above theorem implies that SimCLR suffers an optimization error at least in the order of $O(1/\sqrt{B})$ for the objective's gradient norm. Even with $T\rightarrow\infty$, its optimization error is always dominated by $O(1/\sqrt{B})$. This explains the phenomenon that the performance of SimLCR degrades as the mini-batch size decreases. The above theorem also implies in order to find an $\epsilon$-level stationary solution, i.e., $\E[\|\nabla F(\w_t')\|\leq \epsilon]$, we can set $\eta=O(\epsilon^2)$ and $T=O(1/\epsilon^4)$ and $B=O(1/\epsilon^2)$. All missing proofs can be found in the supplement. 


\subsection{SogCLR and its Convergence for V1 GCO}
To address the issue of SimCLR, in this section we propose a memory-efficient stochastic algorithm for solving~(\ref{eqn:gcl}) without suffering from a large optimization error depending on the batch size. To this end, we decompose the objective function into three terms: 
\begin{align}\label{eqn:egcl}
    &F(\w)  = \E_{\x_i\sim\D, \A, \A'\sim\P}(E(\A(\x_i))^{\top}E(\A'(\x_i)))\\
    &+\frac{\tau}{n}\sum_{\x_i\in\D}\E_{\A}\ln \left(\frac{\varepsilon'}{|\S_i|} + \frac{1}{|\S_i|}g(\w; \x_i, \A, \S_i)\right) + \text{Const},\nonumber
\end{align}
where $\text{Const}$ is a constant that is independent of the model parameters. Below, we let $f(\cdot)=\tau\ln(\varepsilon'/|\S_i|+\cdot)$.

Our algorithm is motivated by the coupled compositional stochastic optimization studied in~\cite{qi2021stochastic} for maximizing Average Precision, whose objective has a form of $\frac{1}{n}\sum_i f(g_i(\w))$ that is similar to the second component in our objective $F(\w)$.  The key idea of the proposed algorithm is to keep track of $\frac{1}{|\S_i|}g(\w; \x_i, \A, \S_i)$ by a scalar, whose averaged error in the long run is diminishing.  However, different from the problem studied in~\cite{qi2021stochastic}, there could be many data augmentations in $\P$. As a result, by maintaining a scalar for each $\x_i\in\D, \A\in\P$, the memory cost is $O(n|\P|)$ which increases as we increase the number of data augmentations and could be very large if $|\P|$ is large. By noting that $\A(\x_i),\forall\A\in\P$ is an augmented data from the same image for different $\A$, we expect that their embedded feature vectors are close in the sense that $\E_{\A,\A',\z}|E(\A(\x_i))^{\top}E(\z) - E(\A'(\x_i))^{\top}E(\z)]|^2\leq \epsilon^2$ for any $\A, \A', \x_i$ and a small value $\epsilon$.  By leveraging this property, we maintain and update a scalar $\u_i$ for each image to track $\frac{1}{|\S_i|}g(\w; \x_i, \A, \S_i)$. 
\begin{algorithm}[t]
    \centering
    \caption{SogCLR}\label{alg:sigclr}
    \begin{algorithmic}[1]
        \STATE \textbf{Input:} $\w_0\in \R^d$, Initialize $\u_0\in \R^n$
        \STATE Draw a batch of $B$ samples denoted by $\mathcal{B}=\{\x_i\}_{i=1}^B$.
        \FOR {$t= 1,\ldots, T$}
          \FOR {$\x_i \in \mathcal B$}
            \STATE Compute $g(\w_t; \x_i, \A,  \B_i)$ and $g(\w_t; \x_i, \A',  \B_i)$ according to~(\ref{eqn:sampleg}) 
            \STATE Update $\u_{i, t}$ according to~(\ref{eqn:u}) 
            \ENDFOR
            \STATE Compute the gradient estimator $\mathbf m_t$ by~(\ref{eqn:m})
             \STATE $\v_{t} = (1-\beta) \v_{t-1} + \beta \mathbf m_t$
            \STATE $\w_{t+1} = \w_t - \eta \v_{t}$ (or use Adam-style update)
        \ENDFOR
    \end{algorithmic}
\end{algorithm}
At the $t$-th iteration, we update $\u_i$ for $\x_i\in\B$ by moving average
\begin{equation}
\label{eqn:u}
    \begin{aligned}
     & \u_{i, t} = (1-\gamma) \u_{i, t-1} \\
     &+ \gamma \frac{1}{2|\B_i|}(g(\w_t; \x_i, \A,  \B_i)+g(\w_t; \x_i, \A',  \B_i)),
           \end{aligned}
\end{equation}
where $\gamma\in(0,1)$.
Then we can compute a stochastic gradient estimator by 
\begin{align}\label{eqn:m}
    \mathbf m_{t}  &=  - \frac{1}{B}\sum_{\x_i\in\B}\nabla (E(\A(\x_i))^{\top}E(\A'(\x_i)))\\
    & + \frac{p_{i,t}}{2|\B_i|}(\nabla g(\w_t; \x_i, \A,  \B_i)+\nabla g(\w_t; \x_i, \A',  \B_i)).\notag
\end{align}
where $p_{i,t}=\frac{\tau}{\varepsilon'/|\S_i| + u_{i,t-1}}=\nabla f(u_{i,t-1})$. 
Finally, we can update the model parameter $\w_{t+1}$ by using a momentum-style update or an Adam-style update. The detailed steps are summarized in Algorithm~\ref{alg:sigclr}, which is referred as SogCLR to emphasize that we aim to optimize the global contrastive objective. 

We note that the memory cost of SogCLR is $O(n+d)$, which is $O(d)$ for over-parameterized deep neural networks with $d\gg n$. The per-iteration complexity of SogCLR is the same as SimCLR.  

Next, we provide a convergence result for SogCLR. 
\begin{thm}\label{thm:2}
Assume that $\E_{\A, \A'}\E_{\z\sim\S_i}|E(\A(\x_i))^{\top}E(\z) - E(\A'(\x_i))^{\top}E(\z)|^2\leq \epsilon^2$ for any $\x_i\in\D$ and  the same conditions as in Theorem~\ref{thm:1} hold, then with  $\gamma\leq  \frac{n}{B}$, and $\eta = O(\min\left\{\beta, \frac{\gamma B}{n}, \frac{1}{L_F}\right\})$, after $T$ iterations, SogCLR ensures that $\E[\|\nabla F(\w_{t'})\|^2]\leq O(\frac{1}{\eta T} + \frac{\beta+\gamma}{B}+ \epsilon^2)$ for a random $t'\in\{1,\ldots, T\}$.  
\end{thm}

{\bf Remark:} The above theorem implies that by setting $\beta = \sqrt{B/T}<1$ and $\gamma=\sqrt{n/T}<1$, then SogCLR's optimization error will converge to the level of $\epsilon$ when $T=O(\max(\frac{n}{B^2\epsilon^4},\frac{1}{B\epsilon^4}))$, i.e., $\E[\|\nabla F(\w_{t'})\|^2]\leq O(\frac{1}{\sqrt{BT}} + \frac{\sqrt{n}}{B\sqrt{T}}+ \epsilon^2)\leq O(\epsilon^2)$. When $\epsilon$ is small enough, the optimization error of SogCLR is negligible. In addition, the analysis also implies that SogCLR enjoys a parallel speed-up, i.e., with a larger mini-batch size $B$ it needs a less number of iterations to converge to a small error. 

One might notice that there are two differences between SogCLR and the update~(\ref{eqn:SimCLR}) for SimCLR. One difference is that SogCLR maintains and updates the $\u$ sequence. The second difference is that SogCLR uses a momentum-style update. We would like to emphasize that the moving average update for $\u_{i,t+1}$ is the key to prove the above result. With this technique, SogCLR is able to leverage the momentum-style update or the Adam-style update to enjoy a small optimization error. Without using the scalars $\u_{i,t+1}$ in computing the gradient estimator, even we use the momentum-style update or the Adam-style update for SimCLR, it still suffers from an optimization error in the order of $O(1/\sqrt{B})$. In particular, we have the following corollary for the optimization error of SimCLR with the momentum-style update. 
\begin{cor}\label{thm:3}
Let us consider the following momentum-style update for SimCLR. 
\begin{align}\label{eqn:SimCLR2}
    &\v_{t}=(1-\beta)\v_{t-1} + \beta \frac{\tau}{B}\sum_{\x_i\in\B}\widehat\nabla L(\x_i, \A, \A')\\
    &\w_{t+1} = \w_t - \eta \v_t.
\end{align}
Assume $F$ is smooth, $g$ is smooth and Lipchitiz continuous, with $\eta\leq O(\beta)$ SimCLR ensures that $\E[\|\nabla F(\w_{t'})\|^2]\leq O(\frac{1}{\eta T} + \frac{1}{\beta T} + \frac{\beta}{B} + \frac{1}{B})$ for a random $t'\in\{1,\ldots, T\}$.  
\end{cor}
{\bf Remark:} The dominating term in the upper bound is still $O(1/B)$ when $T\rightarrow \infty$ and $\beta=O(1/\sqrt{T})$. 

\begin{algorithm}[t]
\caption{PyTorch-style pseudocode for SogCLR }\label{alg:pytorch_sigclr}
\begin{algorithmic}[2]
\vspace{-0.1in}
\begin{lstlisting}[language=Python]
# Note: This is a simplified version of Algorithm 1, we use local u 
# from each augmentation to compute the dynamic contrastive loss 
# instead of aggregated u from all augmentations.
# model: encoder + mlp projectors
# aug: a set of augmentation functions
# tau: temperature
# N: data size 
# ind: indices for images in mini-batch
# u: 1d tensor with shape (N,1) by zero initialization
# g: parameter for maintaining moving averages of u

for ind, img in dataloader:
    x1, x2 = aug(img), aug(img)    # augmentations
    h1, h2 = model(x1), model(x2)  # forward pass
    h1, h2 = h1.norm(dim=1, p=2), h2.norm(dim=1, p=2)
    loss1, u1 = dcl(h1, h2, ind)   # dcl for h1, h2
    loss2, u2 = dcl(h2, h1, ind)   # dcl for h2, h1
    u[ind] = (u1 + u2)/2           # update u
    loss = (loss1 + loss2).mean()  # symmetrized
    loss.backward()             
    update(model.params)           # momentum or adam-style

# dynamic contrastive loss (mini-batch)   
def dcl(h1, h2, ind):
    B = h1.shape[0]  
    labels = cat([one_hot(range(B)), one_hot(range(B))], dim=1)
    logits = cat([dot(h1, h2.T), dot(h1, h1.T)], dim=1)
    neg_logits = exp(logits/tau)*(1-labels) 
    u1 = (1-g) * u[ind] + g*sum(neg_logits, dim=1)/(2(B-1))
    p = (neg_logits/u1).detach() 
    sum_neg_logits = sum(p*logits, dim=1)/(2(B-1))
    normalized_logits = logits - sum_neg_logits
    loss = -sum(labels * normalized_logits, dim=1) 
    return loss, u
\end{lstlisting}
\vspace{-0.2in}
\STATE
\end{algorithmic}
\end{algorithm}

\subsection{SogCLR optimizes V2 Global Contrastive Objective}\label{sec:4}
In this section, we propose another version of the global contrasive objective (V2) and show that SogCLR optimizes the V2 global contrastive objective, which further justifies the proposed algorithm SogCLR. In particular, let us consider the following objective. 
\begin{align}\label{eqn:gcl2}
    &F_{v2}(\w)  = \E_{\x_i\sim\D, \A, \A'\sim\P}(E(\A(\x_i))^{\top}E(\A'(\x_i)))\\
    &+\frac{1}{n}\sum_{\x_i\in\D}\ln \left(\frac{\varepsilon'}{|\S_i|} + \frac{\tau}{|\S_i|}\E_{\A}g(\w; \x_i, \A, \S_i)\right) + \text{Const}.\nonumber
\end{align}
The difference between V2 GCO~(\ref{eqn:gcl2}) and V1 GCO~(\ref{eqn:gcl}) is that the expectation over $\A$ in the second component is moved from the outside of the logarithmic function to the inside. The above objective function can be also explained from the average of individual contrastive loss. To this end, we define the following contrastive loss for each augmented pair $(\A(\x_i), \A'(\x_i))$:
\begin{align}\label{eqn:con2}
L_2(\w; \x_i, \A, \A') = -\ln\frac{\exp(E(\A(\x_i))^{\top}E(\A'(\x_i))/\tau)}{\varepsilon' + \E_{\A}g(\w; \x_i, \A, \S_i)}.
\end{align}
\noindent
Then we have
\begin{align*}
    F_{v2}(\w) = \E_{\x\sim\D, \A, \A'}[\tau L_2(\w; \x_i, \A, \A')].
\end{align*}
Different from $L(\w; \x_i, \A, \A')$, in the definition of $L_2(\w; \x_i, \A, \A')$ the similarity score of a positive pair $(\A(\x_i), \A'(\x_i))$ is contrasted with all possible negative pairs between $\x_i$ and other images. 

We prove that SogCLR indeed converges to a stationary solution to the V2 GCO $F_{v2}(\w)$. Different from $F(\w)$ defined in~(\ref{eqn:gcl}), the update of $u$ of SogCLR can be considered directly as an moving average estimator of $\E_{\A}g(\w; \x_i, \A, \S_i)$ in $F_{v2}(\w)$, which does not involve the error caused by difference between different augmented data. We state the convergence below. 
\begin{thm}\label{thm:4}
Assume the same conditions as in Theorem~\ref{thm:1} hold, then with  $\gamma\leq \frac{n}{B}$, and $\eta = O(\min\left\{\beta, \frac{\gamma B}{n}, \frac{1}{L_F}\right\})$, after $T$ iterations, SogCLR ensures that $\E[\|\nabla F_{v2}(\w_{t'})\|^2]\leq O(\frac{1}{\eta T} + \frac{\beta+\gamma}{B})$ for a random $t'\in\{1,\ldots, T\}$.  
\end{thm}
{\bf Remark:} The above theorem implies that by setting $\beta = \sqrt{B/T}<1$ and $\gamma=\sqrt{n/T}<1$, then SogCLR converges to a stationary solution of $F_{v2}(\w)$ when $T\rightarrow\infty$. 
\section{Extensions}
In this section, we propose the extension of the proposed technique for optimizing other contrastive losses. We note that the large batch size requirement also exists in other contrastive learning methods. Below we consider one task, namely a self-supervised bimodal contrastive learning task.

{\bf Optimizing Two-way Contrastive Objective}. A recent paper~\cite{clip} proposes a bimodal contrastive learning method named CLIP, which uses a two-way contrastive loss to learn both the encoder network for the image and the encoder network for the text. \cite{clip} uses a very large batch size $32,768$ on a self-collected large-scale dataset with 400 million image and text pairs. Inspired by the SogCLR method for optimizing one-way contrastive loss and its promising performance, below we present a similar solution to alleviate the requirement of large batch size for optimizing two-way contrastive loss.   Given a set of image-text pairs $\D=\{(\x_1, \t_1), \ldots, (\x_n, \t_n)\}$. We denote by $E_I$ and $E_T$ the encoder network for the image data and the text data, respectively.  We can consider optimizing a global two-way contrastive loss:
\begin{align*}
F(\w)& = -\frac{\tau}{n}\sum_{i=1}^n\log\frac{\exp(E_I(\x_i)^{\top}E_T(\t_i)/\tau)}{\sum_{\t\in\D}\exp(E_I(\x_i)^{\top}E_T(\t)/\tau)}\\
&-\frac{\tau}{n}\sum_{i=1}^n\log\frac{\exp(E_I(\x_i)^{\top}E_T(\t_i)/\tau)}{\sum_{\x\in\D}\exp(E_I(\x)^{\top}E_T(\t_i)/\tau)}.
\end{align*}
Due to the large size of $\D$, the challenge lies that handling $g(\w; \x_i) = \E_{\t\sim \D}\exp(E_I(\x_i)^{\top}E_T(\t)/\tau))$ and $g(\w; \t_i)=\E_{\x\sim\D}\exp(E_I(\x)^{\top}E_T(\t_i)/\tau))$.  
We propose to compute a stochastic gradient estimator by 
\begin{align*}
    \m_t& =- \frac{1}{B}\sum_{i\in\B}E_I(\x_i)^{\top}E_T(\t_i)+\\
    &\frac{1}{B}\sum_{i\in\B}\left(\frac{\tau}{u_{i,t}^I}\nabla g(\w_t; \x_i, \B)+\frac{\tau}{u_{i,t}^T}\nabla g(\w_t; \t_i, \B)\right)
\end{align*}
where $g(\w; \x_i, \B)$ and $g(\w; \t_i, \B)$ are the mini-batch estimators of $g(\w; \x_i)$ and $g(\w; \t_i)$ respectively. The scalar $u_{i,t+1}^I$ and $u_{i,t+1}^T$ are updated for the sampled data according to 
\begin{align*}
    u_{i,t+1}^I &= (1-\gamma) u_{i,t}^I + \gamma g(\w_t; \x_i, \B)\\
    u_{i,t+1}^T &= (1-\gamma) u_{i,t}^T + \gamma g(\w_t; \t_i, \B).
\end{align*}
Then we can update the model by Adam-style update or momentum-style update. 

\begin{figure*}[h]
\centering
\includegraphics[scale=0.26]{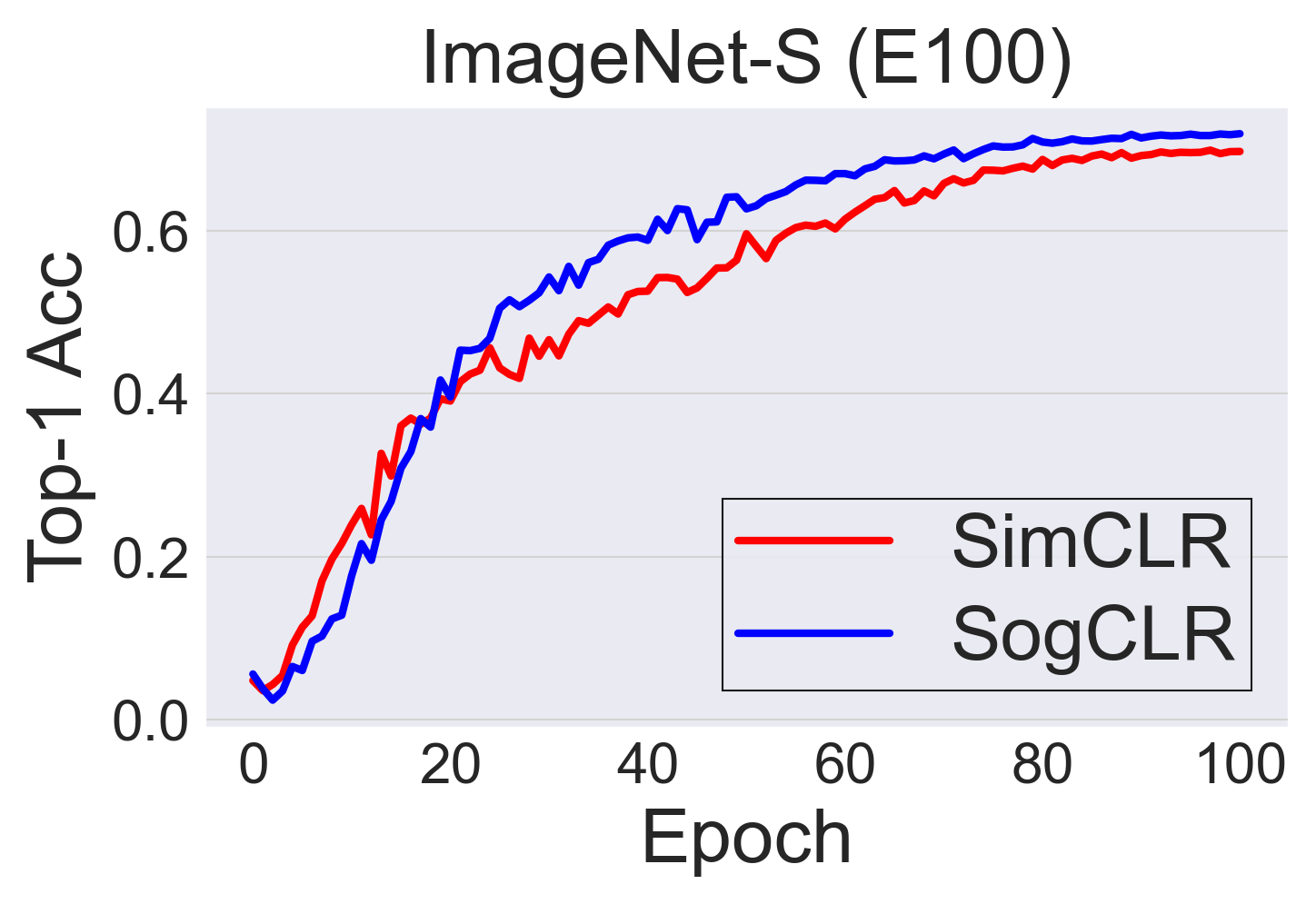}
\includegraphics[scale=0.26]{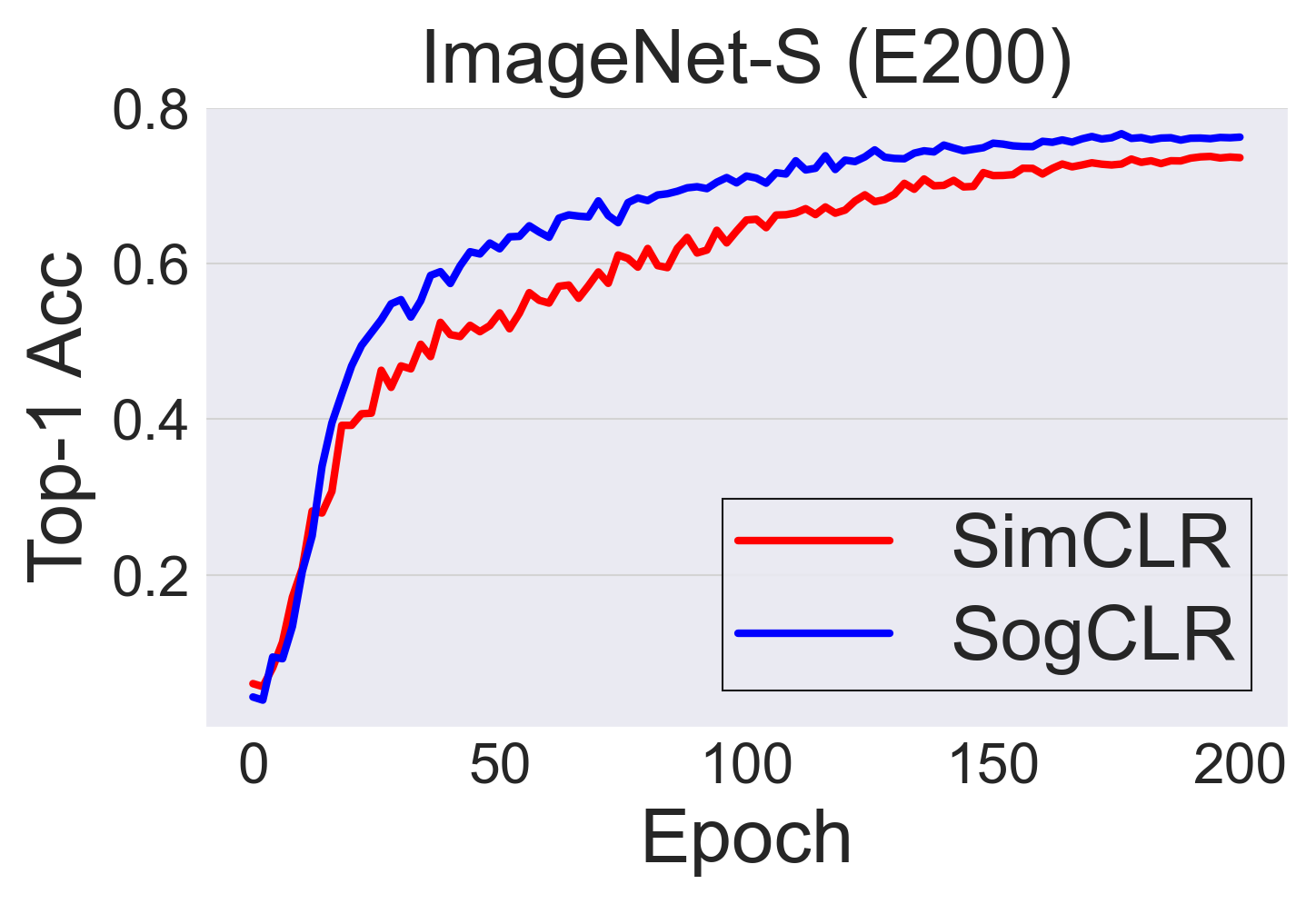}
\includegraphics[scale=0.26]{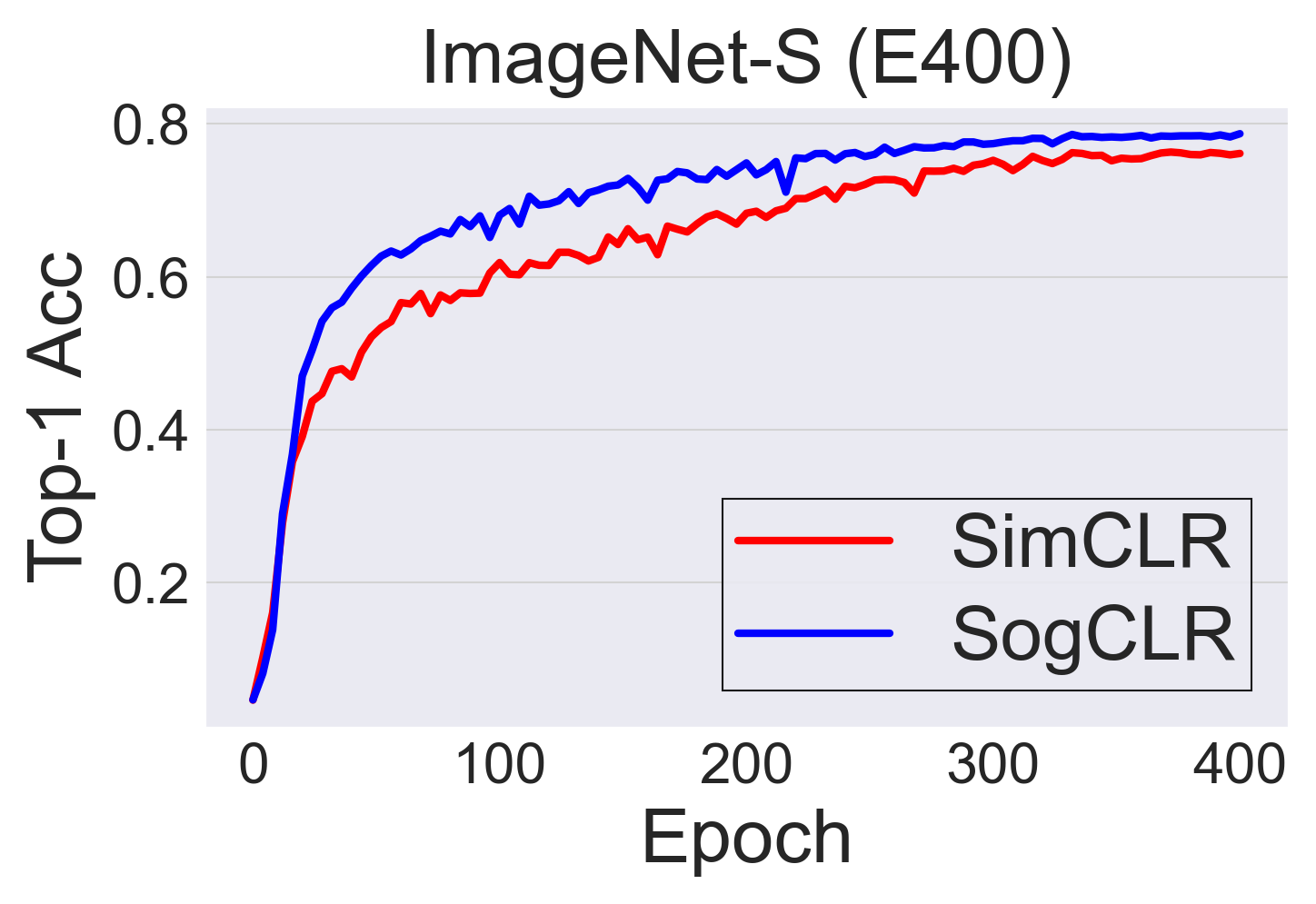}
\includegraphics[scale=0.26]{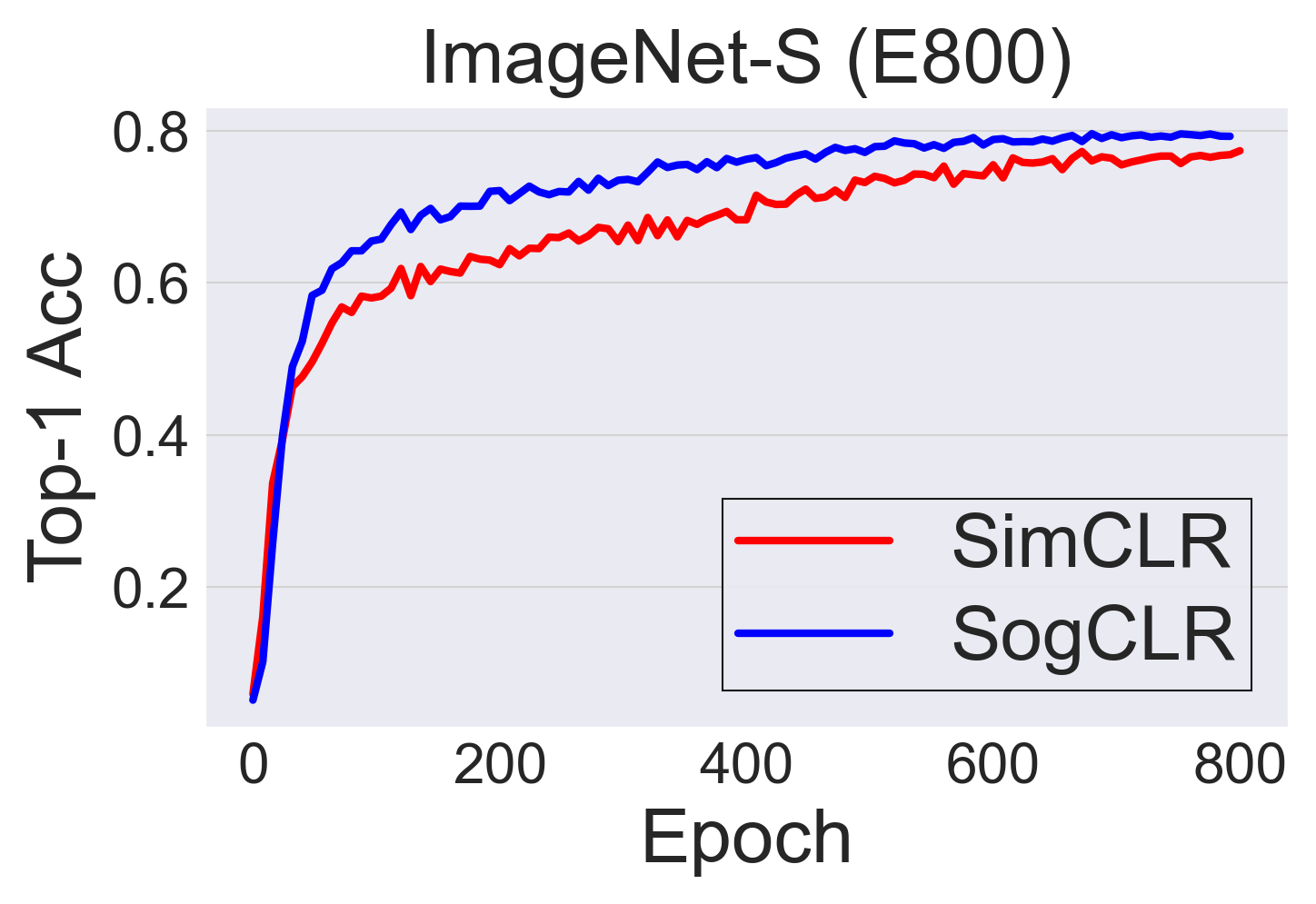}
\includegraphics[scale=0.26]{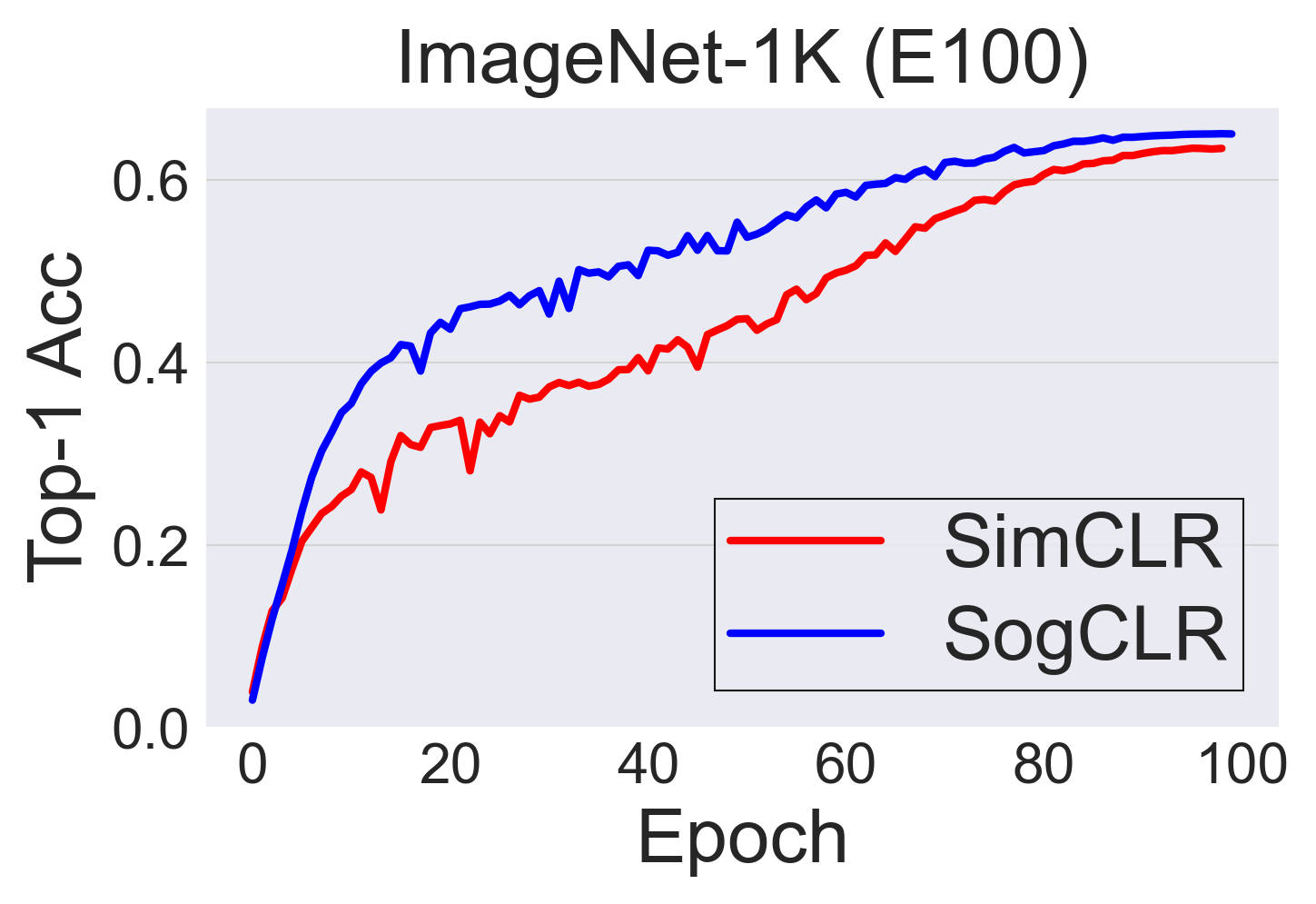}
\includegraphics[scale=0.26]{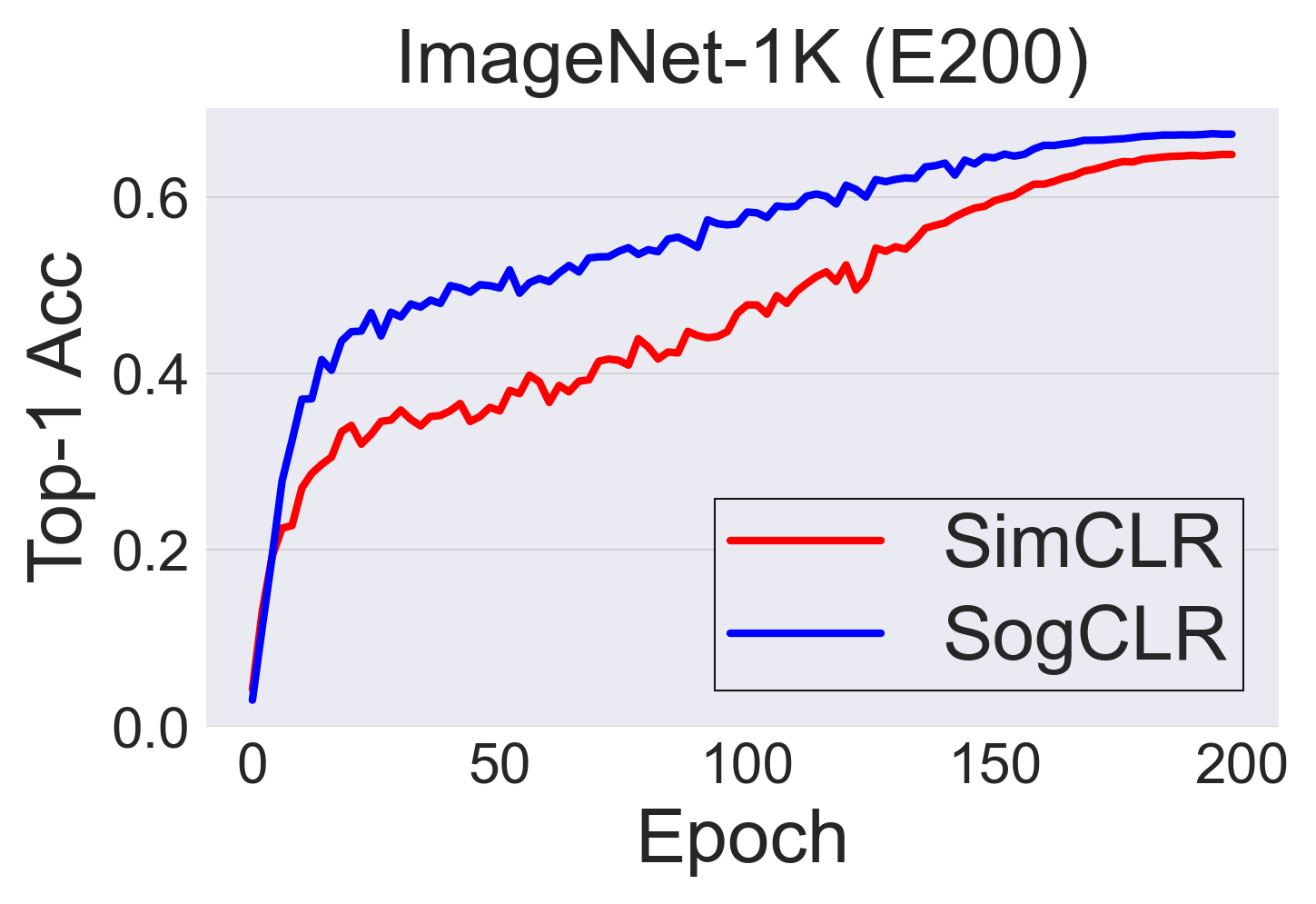}
\includegraphics[scale=0.26]{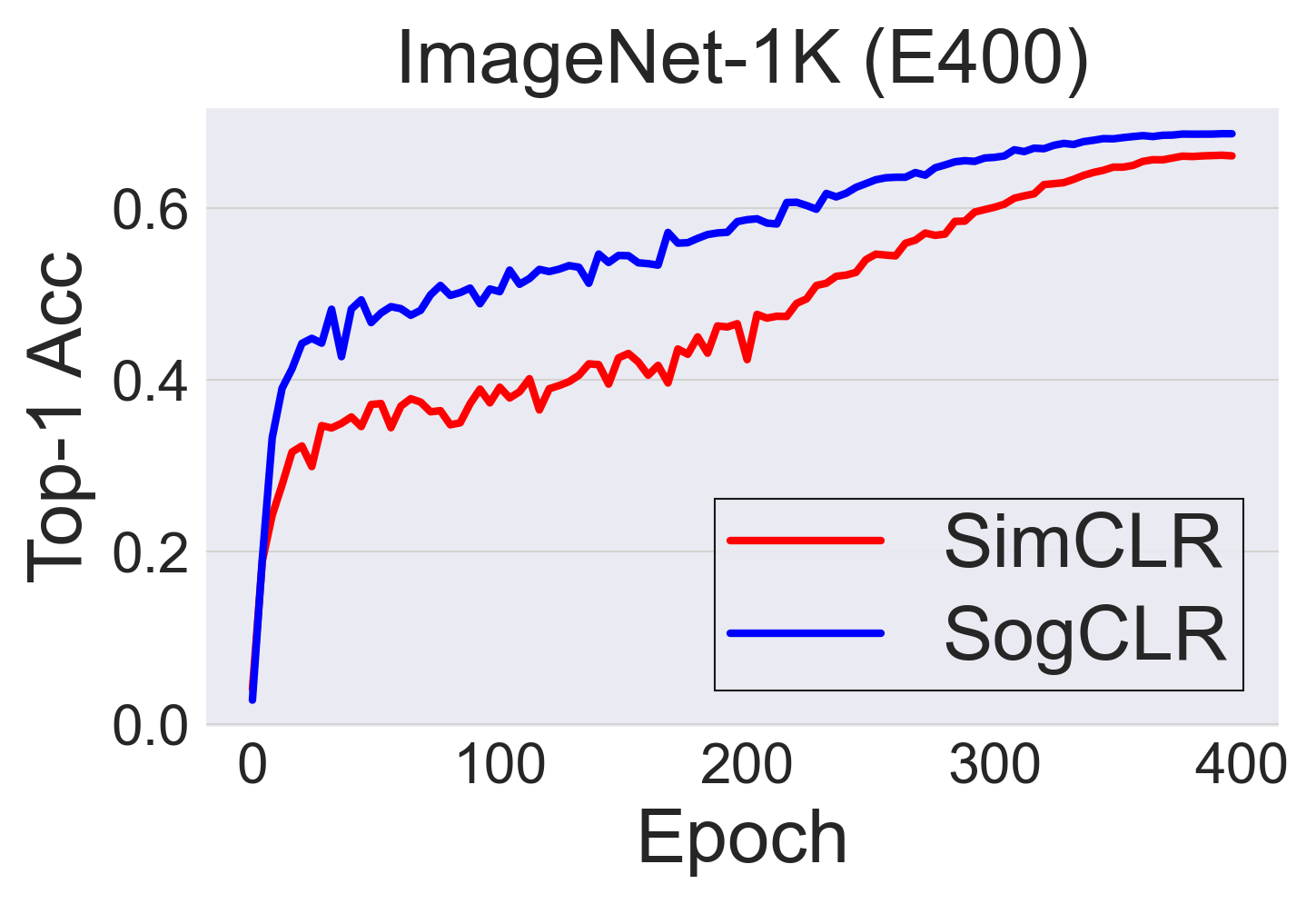}
\includegraphics[scale=0.26]{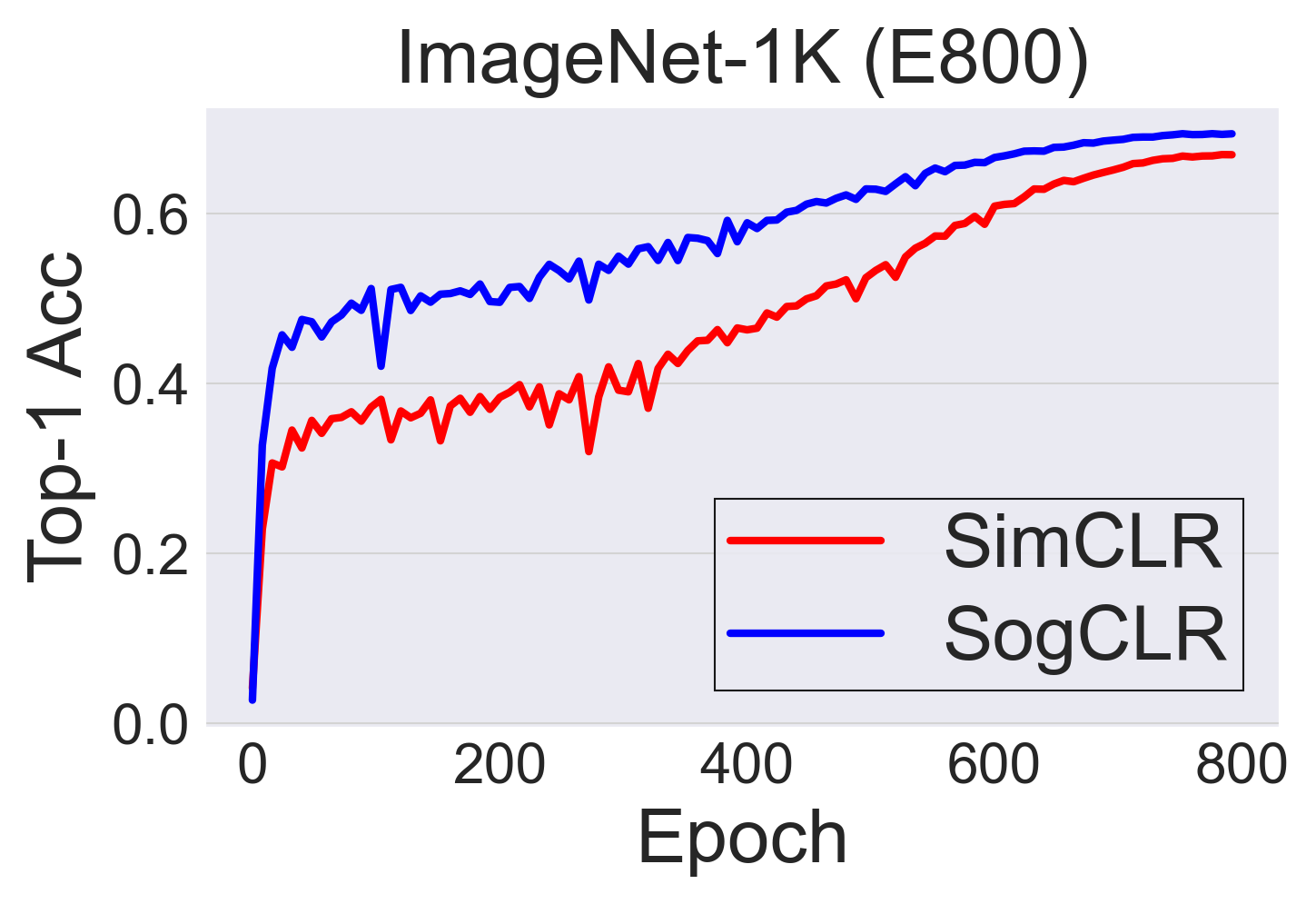}
\vspace{-0.1in}
\caption{Learning curve for top-1 accuracy by linear evaluation on ImageNet-S and ImageNet-1K trained on R50 with batch size of 256. } 
\label{fig:learning_curve_imagenet_100_1000} 
\end{figure*} 

\section{Experiments}
In this section, we compare SogCLR to SimCLR to demonstrate the effectiveness of our optimization method. For a fair comparison, we adopt the same settings as SimCLR to SogCLR unless noted (the main difference is the batch size). It is not our focus to leverage multiple techniques for achieving state-of-the-art performance~\cite{ReLiCv2}. We also compare with the CLIP framework for bimodal contrastive learning. We aim to demonstrate SogCLR can achieve competitive performance when using a smaller batch size. For SimCLR, we run experiments on two scales of ImageNet dataset. The small version is a subset with randomly selected 100 classes (about 128k images) from ImageNet denoted as ImageNet-S~\cite{wu2019large}, and the full version of ImageNet (about 1.2 million images) is denoted as ImageNet-1K~\cite{deng2009imagenet}. For CLIP, we manually construct a text-image pair dataset based on ImageNet-S using the label of each image to construct a text. For the implementations, we follow these open-source repositories~\cite{simclrv1, clip, open_clip} available on Github. Algorithm~\ref{alg:pytorch_sigclr} provides the PyTorch-style pseudo-code of SogCLR. All experiments related to SimCLR are trained on Google Cloud TPUs using 8 to 512 cores depending on model size and batch size. All experiments related to CLIP are trained a NVIDIA V100 GPU with 32GB memory size.  

\subsection{Image Pretraining} \label{sec:simclrv1_imagenet}
\textbf{Experiment setup}. Following previous works~\cite{simclrv1, simclrv2}, we pretrain ResNet-50 \cite{he2016deep} with a 2-layer $128\times128$ projection head on top of backbone encoder. We explore different batch sizes of 128, 256, 512 and different training epochs of 100, 200, 400, 800. We use square root learning rate scaling (0.075$\times$sqrt(BatchSize)) with a cosine decay schedule without restart. We also use learning rate warm-up for 10 epochs, i.e., learning rate is gradually increased to the maximum value. We follow the same image augmentation strategies as in SimCLR~\cite{simclrv1, simclrv2} including random crop, color distortion, and Gaussian blur. We use LARS optimizer~\cite{you2017large} (with a momentum of 0.9 and weight decay of 1e-6) and set temperature($\tau$) to 0.1 by default for all pretraining experiments. For SogCLR in Algorithm~\ref{alg:sigclr}, we tune $\gamma$ in [0.99, 0.9, 0.8, 0.7, 0.6] and initialize sequence $\u_0$ by all zeros. For evaluations, we report performance for linear classifier trained on top of the pretrained encoder on ImageNet validation sets known as linear evaluation~\cite{simclrv1, mocov1, swav, byol}. In particular, we train a linear classier using SGD with Nesterov momentum with a batch size of 4096 and learning rate of 0.1 for 90 epochs. For training, We random crop and resize input images to 224$\times$224. For testing, we apply center crop on input images. 

\textbf{Results}. We report top-1 accuracy by linear evaluation on ImageNet-S and ImageNet-1K under different batch sizes and training epochs in Table~\ref{tab:acc_imagenet_s} and Table~\ref{tab:acc_imagenet}. We can see that SogCLR performs consistently better than SimCLR under all settings on two datasets. SogCLR achieves 3.9\%, 2.3\%, 1.9\% average improvements on ImageNet-S and achieves 2.8\%, 2.3\%, 1.3\% average improvements on ImageNet-1K with batch size of 128, 256, 512, respectively. In particular, we achieve 69.4\% top-1 accuracy using batch size of 256, which is better than original SimCLR’s large-batch (e.g., 4096, 8192) results at 69.1\% under the same number of epochs. In addition, we compare the convergence speed of SogCLR with SimCLR using the same batch size of 256 with different number of epochs on ImageNet-S and ImageNet-1K as shown in Figure~\ref{fig:learning_curve_imagenet_100_1000}. The results indicate that our algorithm converges faster in terms of number of epochs using small batch sizes. 
\begin{table}[h]
\centering
\caption{Linear evaluation (top-1 accuracy) under different batch sizes and training epoch on ResNet-50 and ImageNet-S.}
\label{tab:acc_imagenet_s}
\scalebox{0.95}{
\begin{tabular}{cccccc}
\hline
Method & BatchSize\textbackslash{}Epoch & 100 & 200 & 400 & 800 \\ \hline
\multicolumn{1}{c}{SimCLR} & 128 & 68.5 & 72.7 & 75.7 & 75.7 \\ 
\multicolumn{1}{c}{SogCLR} & 128 & \textbf{72.2} & \textbf{76.7} & \textbf{79.3} & \textbf{80.1} \\ \hline
\multicolumn{1}{c}{SimCLR} & 256 & 69.7 & 73.6 & 76.1 & 77.4 \\ 
\multicolumn{1}{c}{SogCLR} & 256 & \textbf{71.8} & \textbf{76.3} & \textbf{78.7} & \textbf{79.4} \\ \hline
\multicolumn{1}{c}{SimCLR} & 512 & 70.9 & 74.1 & 75.9 & 76.3 \\
\multicolumn{1}{c}{SogCLR} & 512 & \textbf{71.8} & \textbf{75.8} & \textbf{78.2} & \textbf{79.4} \\ \hline
\end{tabular}}
\end{table}
\begin{table}[h]
\centering
\caption{Linear evaluation (top-1 accuracy) under different batch sizes and training epoch on ResNet-50 and ImageNet-1K.}
\label{tab:acc_imagenet}
\scalebox{0.95}{
\begin{tabular}{cccccc}
\hline
Method & BatchSize\textbackslash{}Epoch & 100 & 200 & 400 & 800 \\ \hline
\multicolumn{1}{c}{SimCLR} & 128 & 62.6 & 64.0 & 64.1 & 64.5 \\ 
\multicolumn{1}{c}{SogCLR} & 128 & \textbf{64.9} & \textbf{66.2} & \textbf{67.4} & \textbf{67.9} \\ \hline
\multicolumn{1}{c}{SimCLR} & 256 & 62.8 & 64.3 & 65.7 & 66.5 \\
\multicolumn{1}{c}{SogCLR} & 256 & \textbf{65.2} & \textbf{67.1} & \textbf{68.7} & \textbf{69.4} \\ \hline
\multicolumn{1}{c}{SimCLR} & 512 & 63.8 & 65.6 & 66.7 & 67.4 \\ 
\multicolumn{1}{c}{SogCLR} & 512 & \textbf{65.0} & \textbf{67.2} & \textbf{68.8} & \textbf{69.6} \\ \hline
\end{tabular}}
\end{table}

\subsection{Vision and Language Pretraining}
\textbf{Experiment Setup}. In this section, we aim to demonstrate our algorithm can also be applied to solve bi-modal self-supervised problems. We study a popular vision and language pretraining framework, i.e., CLIP~\cite{clip}. CLIP consists of two parts: vision encoder (e.g., CNN, transformer) and text encoder (e.g., transformer). The original CLIP is pretrained on a large dataset with 400 million image-text pairs to achieve competitive performance against supervised baseline. Here, we are not aiming to achieve the best performance but to study and understand the limits of this framework. Thus, we use the modified CLIP consisting of a modified ResNet-50 and a small vision transformer(ViT)~\cite{Dosovitskiy2021AnII}, denoted as CLIP-S. The detailed configuration can be found in Appendix. We use template \textit{"This is a photo of [CLASS]"} to generate the text caption for each image based on ImageNet-S. For training, we use batch size of 128 and 256 to train the models for 30 and 60 epochs. We use warm-up strategy for 1000 iterations to increase learning rate to the maximum value of 0.001 and then decrease it by a cosine decay scheduler. We use Adam-W optimizer~\cite{loshchilov2017decoupled} with the weight decay of 0.1. We set temperature to a fixed value for 0.07 for SogCLR and CLIP. Similar to SimCLR, we tune $\gamma=[0.6\sim 0.99]$ and set $\u_0$ to zeros for SogCLR. For evaluations, we perform zero-shot evaluation on ImageNet-S validation set using the ensemble results of 80 different prompt templates~\cite{clip}. The validation results are presented in Table~\ref{tab:clip_zero_shot}.

\textbf{Results}. We report zero-shot evaluation accuracy of CLIP-S in Table~\ref{tab:clip_zero_shot}. The results indicate that CLIP-S trained by SogCLR performs better than CLIP-S trained by standard InfoNCE loss. In addition, we observe that InfoNCE suffers from 4\% performance drop for training 60 epochs. In contrast, SogCLR has a much more stable performance for longer training and achieves over 1\% improvement on zero-shot evaluation accuracy. In addition, we also find that CLIP with SogCLR is much more robust to the change of batch size while CLIP with InfoNCE drops more than 1\% when switching batch size from 256 to 128. 
\begin{table}[t]
\centering
\caption{Top-1 linear evaluation under different batch sizes for bimodal learning on ImageNet-S.}
\label{tab:clip_zero_shot}
\scalebox{0.95}{
\begin{tabular}{cccc}
\hline
Method   & BatchSize\textbackslash{}Epoch & 30            & 60            \\ \hline
CLIP-S (InfoNCE) & 128                            & 67.7          & 63.4          \\ 
CLIP-S (SogCLR)  & 128                            & \textbf{69.5} & \textbf{71.3} \\ \hline
CLIP-S (InfoNCE) & 256                            & 69.0          & 64.9          \\ 
CLIP-S (SogCLR)  & 256                            & \textbf{69.4} & \textbf{70.1} \\ \hline
\end{tabular}}
\end{table}

\subsection{Ablation Studies}
\textbf{Verification of algorithmic design and theory.}
We validate (i) using the momentum update for $u_{t+1}$ (i.e., $\gamma < 1$) is better than without using momentum update ($\gamma = 1$). (ii) $\E_{\A, \A', \z}|E(\A(\x_i))^{\top}E(\z) - E(\A'(\x_i))^{\top}E(\z)|^2\leq \epsilon^2$ in Theorem 2 holds with a small $\epsilon^2$. In other words, we expect the similarity between the representations of different augmented samples are close. For (i), we train ResNet-50 with batch size of 256 for 100, 200, 400, 800 epochs. We tune the $\gamma$ in [0.6, 0.7, 0.8, 0.9, 0.99]. The results are summarized in the Table~\ref{tab:4} in Appendix. The results indicate that models with $\gamma=0.7\sim0.8$ achieve the best performance. For (ii), we use the models trained with batch size 256 at the checkpoints of 100th, 200th, 400th, 800th epoch to compute  $\E_{\A,\A',\z}[|\A(\x_i)^{\top}\z-\A'(\x_i)^{\top}\z|^2]$ on ImageNet-S, where the expectation is approximated by the Monte Carlo method. We show the histograms of this quantity for all images $\x_i$  in Figure~\ref{fig:hist_error}, which suggests that all data samples satisfy the above condition in Theorem 2 for some small $\epsilon$. 
\begin{figure}[t]
\centering
\includegraphics[scale=0.5]{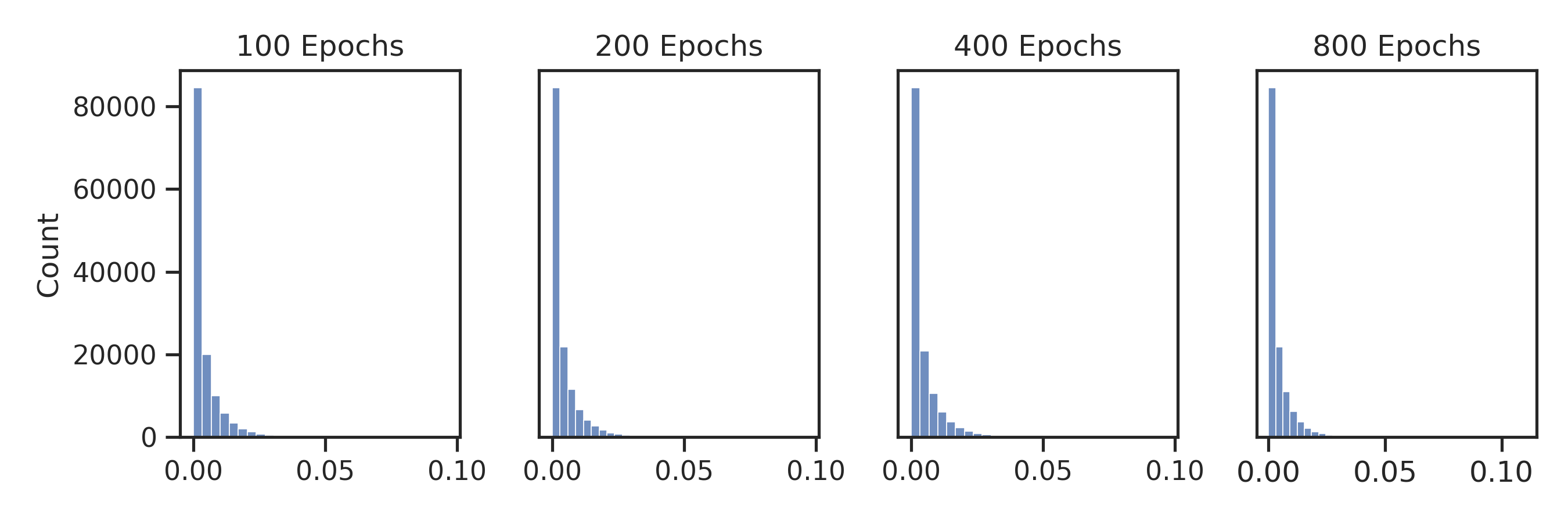}
\vspace{-0.2in}
\caption{Histogram for difference of learned features between different augmented samples. X-axis is $\E_{\A,\A',\z}|E(\A(\x_i))^{\top}E(\z) - E(\A'(\x_i))^{\top}E(\z)|^2$ and Y-axis denotes the count number.} 
\label{fig:hist_error} 
\end{figure} 

\textbf{Impact of batch size}. Since SimCLR suffers from the performance drop due to small batch sizes. Here, we show SogCLR is robust to smaller batch sizes. To verify this hypothesis, we train SogCLR using batch sizes varying from 128 to 8192 with ResNet-50 on ImageNet-1K. We set a fixed $\gamma=0.8$. We directly compare the results taken from Table B.1 in \cite{simclrv1} using the same settings. As shown in Figure \ref{fig:impact_of_batch_size}, the performance of SimCLR drops quickly as the decrease of batch size.  As a comparison, SogCLR remains stable with batch sizes from 8192 to 256 and there is a small drop for batch size of 128. Overall, SogCLR demonstrates the robustness to different batch sizes. This result is consistent with our theory.

\noindent 
\textbf{Different network encoders}. To verify the effectiveness of the proposed method, we further evaluate it on different network encoders. To this end, we train ResNet models by varying widths. We train ResNet-50 ($2\times, 4\times$) using batch size of 512 for 800 epochs. We set $\gamma=0.8$. For baselines, we use the batch size of 4096 to train models for a total of 1000 epochs. The results are summarized in the table below.
\begin{table}[h]
\centering
\caption{Performance with ResNet-50 ($2\times, 4\times$) on ImageNet-1K.}
\label{tab:my-table}
\scalebox{0.95}{
\begin{tabular}{cccccc}
\hline
{Method} & {Encoder} & {Params} & {Batch} & {Top1} & {Top5} \\ \hline
SimCLR & R50 ($2\times$) & 94M & 4096 & 74.2 & 92.0 \\
SogCLR & R50 ($2\times$)   & 94M & 512 & \textbf{74.6} & 92.1 \\ \hline
SimCLR & R50 ($4\times$) & 375M & 4096 & 76.5 & 93.2  \\ 
SogCLR & R50 ($4\times$)    & 375M & 512 & \textbf{76.7} & 93.1 \\\hline
\end{tabular}}
\end{table} 
When using ResNet-50 ($2\times, 4\times$), we are able to achieve 74.6\% and 76.7\% top-1 linear evaluation accuracy, which are better than SimCLR's results trained with a larger batch size of 4096 and a large epoch number. 

\subsection{Comparison with Small Batch Size Methods}
We conduct experiments to compare SogCLR with other two InfoNCE-based small-batch training methods, e.g., FlatNCE~\cite{flatnce} and SiMo~\cite{eqco}. We train ResNet-50 with a 2-layer nonlinear projection head on ImageNet-S using a batch size of 256 for 100, 200, 400, 800 epochs. For SiMo, we set $\alpha=65536$. The results are summarized in Table~\ref{tab:small_batch_comparision}. We observe that all methods outperform SimCLR for 100 and 200 epochs and the improvements for FlatNCE and SiMo seem to disappear when reaching 800 epochs. In contrast, SogCLR performs consistently better. 
\begin{table}[h]
\centering
\caption{{Comparison of small-batch training approaches.}}
\scalebox{0.95}{
\begin{tabular}{cccccc}
\hline
Method         & Batch Size\textbackslash{}Epochs & 100           & 200           & 400           & 800           \\ \hline
SimCLR       & 256                              & 69.7          & 73.6          & 76.1          & 77.4          \\ 
FlatNCE        & 256                              & 71.5          & 75.5          & 76.7          & 77.8          \\ 
SiMo & 256                              & 71.5          & 75.0          & 76.8          & 78.2          \\ 
SogCLR        & 256                              & \textbf{71.9} & \textbf{76.3} & \textbf{78.7} & \textbf{79.4} \\ \hline
\end{tabular}}\label{tab:small_batch_comparision}
\end{table}

\subsection{Combining with Other Useful Tricks}
 We explore two commonly used techniques to boost the performance in our framework, namely, multi-layer projection head~\cite{simclrv1} and multi-crop augmentation~\cite{swav}, and incorporate them into SogCLR. We use a 4-layer MLP projection head with 128 neurons for each layer on the top of ResNet-50 encoder, and use a multi-crop strategy by using 4 crops of size $160\times 160$ and 2 crops of size $96 \times 96$.  SogCLR achieves 72.5\% top-1 linear evaluation accuracy with these two tricks and a batch size of 256 and 800 epochs, which is reported in Table~\ref{tab:ssl_all}. Momentum encoder is introduced by MoCo~\cite{mocov1}. We expect the performance of SogCLR can be further improved by incorporating other techniques, e.g., InfoMin augmentation. We also experiment with different combinations of projection heads and multi-crop data augmentations summarized in Table~\ref{tab:tricks}. 

\begin{table*}[h]
\caption{Comparison of different InfoNCE-loss based contrastive learning methods and their top-1 linear evaluation accuracy on ImageNet-1K.}
\centering
\scalebox{0.88}{
\begin{tabular}{cccccc|c}
\hline
Method & Batch Size & \begin{tabular}[c]{@{}c@{}}Memory \\ Bank\end{tabular} & \begin{tabular}[c]{@{}c@{}}Momentum \\ Encoder\end{tabular} & \begin{tabular}[c]{@{}c@{}}Other \\ Tricks\end{tabular} & Convergence & \multicolumn{1}{c}{Top1 Acc.} \\ \hline
SimCLR~\cite{simclrv1}  & Large-batch & No & No & Strong Aug. &No  & 66.5 \\ 
NNCLR~\cite{dwibedi2021little} & Large-batch & No & No & Nearest Neighbors & No & 68.7 \\
SiMo~\cite{eqco} & Small-batch & No & Yes & Margin Trick & No &72.1\\
\hline
MoCov2~\cite{mocov2}  & Small-batch & Yes & Yes & Strong Aug. & No  & 71.1 \\ 
InfoMin~\cite{tian2020makes} & Small-batch & Yes & Yes & InfoMin Aug.  & No &73.0 \\
\hline
SogCLR (Ours) & Small-batch & No  & No & GC Optimization & Yes  & 72.5 \\ 
\hline
\end{tabular}}\label{tab:ssl_all}
\end{table*}

\section{Conclusion}
In this paper, we have examined the large batch size issue in the training of contrastive self-supervised learning from an optimization perspective. To address this issue, we have proposed a global contrastive objective and an efficient stochastic algorithm with provable convergence guarantee. Our analysis also exhibits why existing methods such as SimCLR require a large batch size for ensuring the optimization error to be small. For future work, we plan to incorporate more advanced techniques into the proposed method to further improve the performance. 

\section*{Acknowledgements}
We would like to thank Quanqi Hu for the help on the proofs.  We also thank anonymous reviewers for their constructive comments. This work is partially supported by NSF Grants 2110545, 1844403, and 1733742.


\bibliographystyle{plain}
\bibliography{reference}

\begin{thebibliography}{10}

\bibitem{swav}
Mathilde Caron, Ishan Misra, Julien Mairal, Priya Goyal, Piotr Bojanowski, and
  Armand Joulin.
\newblock Unsupervised learning of visual features by contrasting cluster
  assignments.
\newblock {\em arXiv preprint arXiv:2006.09882}, 2020.

\bibitem{chen2021simpler}
Junya Chen, Zhe Gan, Xuan Li, Qing Guo, Liqun Chen, Shuyang Gao, Tagyoung
  Chung, Yi~Xu, Belinda Zeng, Wenlian Lu, et~al.
\newblock Simpler, faster, stronger: Breaking the log-k curse on contrastive
  learners with flatnce.
\newblock {\em arXiv preprint arXiv:2107.01152}, 2021.

\bibitem{flatnce}
Junya Chen, Zhe Gan, Xuan Li, Qing Guo, Liqun Chen, Shuyang Gao, Tagyoung
  Chung, Yi~Xu, Belinda Zeng, Wenlian Lu, et~al.
\newblock Simpler, faster, stronger: Breaking the log-k curse on contrastive
  learners with flatnce.
\newblock {\em arXiv preprint arXiv:2107.01152}, 2021.

\bibitem{simclrv1}
Ting Chen, Simon Kornblith, Mohammad Norouzi, and Geoffrey Hinton.
\newblock A simple framework for contrastive learning of visual
  representations.
\newblock In {\em International conference on machine learning}, pages
  1597--1607. PMLR, 2020.

\bibitem{simclrv2}
Ting Chen, Simon Kornblith, Kevin Swersky, Mohammad Norouzi, and Geoffrey
  Hinton.
\newblock Big self-supervised models are strong semi-supervised learners.
\newblock {\em arXiv preprint arXiv:2006.10029}, 2020.

\bibitem{mocov2}
Xinlei Chen, Haoqi Fan, Ross Girshick, and Kaiming He.
\newblock Improved baselines with momentum contrastive learning.
\newblock {\em arXiv preprint arXiv:2003.04297}, 2020.

\bibitem{DBLP:conf/cvpr/ChenH21}
Xinlei Chen and Kaiming He.
\newblock Exploring simple siamese representation learning.
\newblock In {\em {IEEE} Conference on Computer Vision and Pattern Recognition,
  {CVPR} 2021, virtual, June 19-25, 2021}, pages 15750--15758. Computer Vision
  Foundation / {IEEE}, 2021.

\bibitem{mocov3}
Xinlei Chen, Saining Xie, and Kaiming He.
\newblock An empirical study of training self-supervised vision transformers.
\newblock {\em arXiv preprint arXiv:2104.02057}, 2021.

\bibitem{chopra2005learning}
Sumit Chopra, Raia Hadsell, and Yann LeCun.
\newblock Learning a similarity metric discriminatively, with application to
  face verification.
\newblock In {\em 2005 IEEE Computer Society Conference on Computer Vision and
  Pattern Recognition (CVPR'05)}, volume~1, pages 539--546. IEEE, 2005.

\bibitem{chuang2020debiased}
Ching-Yao Chuang, Joshua Robinson, Yen-Chen Lin, Antonio Torralba, and Stefanie
  Jegelka.
\newblock Debiased contrastive learning.
\newblock {\em Advances in neural information processing systems},
  33:8765--8775, 2020.

\bibitem{deng2009imagenet}
Jia Deng, Wei Dong, Richard Socher, Li-Jia Li, Kai Li, and Li~Fei-Fei.
\newblock Imagenet: A large-scale hierarchical image database.
\newblock In {\em 2009 IEEE conference on computer vision and pattern
  recognition}, pages 248--255. Ieee, 2009.

\bibitem{devlin2018bert}
Jacob Devlin, Ming-Wei Chang, Kenton Lee, and Kristina Toutanova.
\newblock Bert: Pre-training of deep bidirectional transformers for language
  understanding.
\newblock {\em arXiv preprint arXiv:1810.04805}, 2018.

\bibitem{dosovitskiy2020image}
Alexey Dosovitskiy, Lucas Beyer, Alexander Kolesnikov, Dirk Weissenborn,
  Xiaohua Zhai, Thomas Unterthiner, Mostafa Dehghani, Matthias Minderer, Georg
  Heigold, Sylvain Gelly, et~al.
\newblock An image is worth 16x16 words: Transformers for image recognition at
  scale.
\newblock {\em arXiv preprint arXiv:2010.11929}, 2020.

\bibitem{Dosovitskiy2021AnII}
Alexey Dosovitskiy, Lucas Beyer, Alexander Kolesnikov, Dirk Weissenborn,
  Xiaohua Zhai, Thomas Unterthiner, Mostafa Dehghani, Matthias Minderer, Georg
  Heigold, Sylvain Gelly, Jakob Uszkoreit, and Neil Houlsby.
\newblock An image is worth 16x16 words: Transformers for image recognition at
  scale.
\newblock {\em ArXiv}, abs/2010.11929, 2021.

\bibitem{dwibedi2021little}
Debidatta Dwibedi, Yusuf Aytar, Jonathan Tompson, Pierre Sermanet, and Andrew
  Zisserman.
\newblock With a little help from my friends: Nearest-neighbor contrastive
  learning of visual representations.
\newblock In {\em Proceedings of the IEEE/CVF International Conference on
  Computer Vision}, pages 9588--9597, 2021.

\bibitem{byol}
Jean-Bastien Grill, Florian Strub, Florent Altch{\'e}, Corentin Tallec,
  Pierre~H Richemond, Elena Buchatskaya, Carl Doersch, Bernardo~Avila Pires,
  Zhaohan~Daniel Guo, Mohammad~Gheshlaghi Azar, et~al.
\newblock Bootstrap your own latent: A new approach to self-supervised
  learning.
\newblock {\em arXiv preprint arXiv:2006.07733}, 2020.

\bibitem{guo2021stochastic}
Zhishuai Guo, Yi~Xu, Wotao Yin, Rong Jin, and Tianbao Yang.
\newblock On stochastic moving-average estimators for non-convex optimization.
\newblock {\em arXiv preprint arXiv:2104.14840}, 2021.

\bibitem{gutmann2010noise}
Michael Gutmann and Aapo Hyv{\"a}rinen.
\newblock Noise-contrastive estimation: A new estimation principle for
  unnormalized statistical models.
\newblock In {\em Proceedings of the thirteenth international conference on
  artificial intelligence and statistics}, pages 297--304. JMLR Workshop and
  Conference Proceedings, 2010.

\bibitem{10.1109/CVPR.2006.100}
Raia Hadsell, Sumit Chopra, and Yann LeCun.
\newblock Dimensionality reduction by learning an invariant mapping.
\newblock In {\em Proceedings of the 2006 IEEE Computer Society Conference on
  Computer Vision and Pattern Recognition - Volume 2}, CVPR '06, page
  1735–1742, USA, 2006. IEEE Computer Society.

\bibitem{mocov1}
Kaiming He, Haoqi Fan, Yuxin Wu, Saining Xie, and Ross Girshick.
\newblock Momentum contrast for unsupervised visual representation learning.
\newblock In {\em Proceedings of the IEEE/CVF Conference on Computer Vision and
  Pattern Recognition}, pages 9729--9738, 2020.

\bibitem{he2016deep}
Kaiming He, Xiangyu Zhang, Shaoqing Ren, and Jian Sun.
\newblock Deep residual learning for image recognition.
\newblock In {\em Proceedings of the IEEE conference on computer vision and
  pattern recognition}, pages 770--778, 2016.

\bibitem{open_clip}
Gabriel Ilharco, Mitchell Wortsman, Nicholas Carlini, Rohan Taori, Achal Dave,
  Vaishaal Shankar, Hongseok Namkoong, John Miller, Hannaneh Hajishirzi, Ali
  Farhadi, and Ludwig Schmidt.
\newblock Openclip, July 2021.
\newblock If you use this software, please cite it as below.

\bibitem{khosla2020supervised}
Prannay Khosla, Piotr Teterwak, Chen Wang, Aaron Sarna, Yonglong Tian, Phillip
  Isola, Aaron Maschinot, Ce~Liu, and Dilip Krishnan.
\newblock Supervised contrastive learning.
\newblock {\em arXiv preprint arXiv:2004.11362}, 2020.

\bibitem{lan2019albert}
Zhenzhong Lan, Mingda Chen, Sebastian Goodman, Kevin Gimpel, Piyush Sharma, and
  Radu Soricut.
\newblock Albert: A lite bert for self-supervised learning of language
  representations.
\newblock {\em arXiv preprint arXiv:1909.11942}, 2019.

\bibitem{li2021efficient}
Chunyuan Li, Jianwei Yang, Pengchuan Zhang, Mei Gao, Bin Xiao, Xiyang Dai,
  Lu~Yuan, and Jianfeng Gao.
\newblock Efficient self-supervised vision transformers for representation
  learning.
\newblock {\em arXiv preprint arXiv:2106.09785}, 2021.

\bibitem{li2020prototypical}
Junnan Li, Pan Zhou, Caiming Xiong, and Steven~CH Hoi.
\newblock Prototypical contrastive learning of unsupervised representations.
\newblock {\em arXiv preprint arXiv:2005.04966}, 2020.

\bibitem{liu2021swin}
Ze~Liu, Yutong Lin, Yue Cao, Han Hu, Yixuan Wei, Zheng Zhang, Stephen Lin, and
  Baining Guo.
\newblock Swin transformer: Hierarchical vision transformer using shifted
  windows.
\newblock {\em arXiv preprint arXiv:2103.14030}, 2021.

\bibitem{loshchilov2017decoupled}
Ilya Loshchilov and Frank Hutter.
\newblock Decoupled weight decay regularization.
\newblock {\em arXiv preprint arXiv:1711.05101}, 2017.

\bibitem{mikolov2013efficient}
Tomas Mikolov, Kai Chen, Greg Corrado, and Jeffrey Dean.
\newblock Efficient estimation of word representations in vector space.
\newblock {\em arXiv preprint arXiv:1301.3781}, 2013.

\bibitem{ReLiCv1}
Jovana Mitrovic, Brian McWilliams, Jacob Walker, Lars Buesing, and Charles
  Blundell.
\newblock Representation learning via invariant causal mechanisms.
\newblock {\em arXiv preprint arXiv:2010.07922}, 2020.

\bibitem{oord2018representation}
Aaron van~den Oord, Yazhe Li, and Oriol Vinyals.
\newblock Representation learning with contrastive predictive coding.
\newblock {\em arXiv preprint arXiv:1807.03748}, 2018.

\bibitem{qi2021stochastic}
Qi~Qi, Youzhi Luo, Zhao Xu, Shuiwang Ji, and Tianbao Yang.
\newblock Stochastic optimization of areas under precision-recall curves with
  provable convergence.
\newblock {\em Advances in Neural Information Processing Systems}, 34, 2021.

\bibitem{Qian2021SpatiotemporalCV}
Rui Qian, Tianjian Meng, Boqing Gong, Ming-Hsuan Yang, H.~Wang, Serge~J.
  Belongie, and Yin Cui.
\newblock Spatiotemporal contrastive video representation learning.
\newblock {\em 2021 IEEE/CVF Conference on Computer Vision and Pattern
  Recognition (CVPR)}, pages 6960--6970, 2021.

\bibitem{clip}
Alec Radford, Jong~Wook Kim, Chris Hallacy, Aditya Ramesh, Gabriel Goh,
  Sandhini Agarwal, Girish Sastry, Amanda Askell, Pamela Mishkin, Jack Clark,
  et~al.
\newblock Learning transferable visual models from natural language
  supervision.
\newblock {\em arXiv preprint arXiv:2103.00020}, 2021.

\bibitem{sohn2016improved}
Kihyuk Sohn.
\newblock Improved deep metric learning with multi-class n-pair loss objective.
\newblock In {\em Advances in neural information processing systems}, pages
  1857--1865, 2016.

\bibitem{tian2020makes}
Yonglong Tian, Chen Sun, Ben Poole, Dilip Krishnan, Cordelia Schmid, and
  Phillip Isola.
\newblock What makes for good views for contrastive learning?
\newblock {\em Advances in Neural Information Processing Systems},
  33:6827--6839, 2020.

\bibitem{ReLiCv2}
Nenad Tomasev, Ioana Bica, Brian McWilliams, Lars Buesing, Razvan Pascanu,
  Charles Blundell, and Jovana Mitrovic.
\newblock Pushing the limits of self-supervised resnets: Can we outperform
  supervised learning without labels on imagenet?
\newblock {\em arXiv preprint arXiv:2201.05119}, 2022.

\bibitem{wang2020understanding}
Tongzhou Wang and Phillip Isola.
\newblock Understanding contrastive representation learning through alignment
  and uniformity on the hypersphere.
\newblock In {\em International Conference on Machine Learning}, pages
  9929--9939. PMLR, 2020.

\bibitem{wu2019large}
Yue Wu, Yinpeng Chen, Lijuan Wang, Yuancheng Ye, Zicheng Liu, Yandong Guo, and
  Yun Fu.
\newblock Large scale incremental learning.
\newblock In {\em Proceedings of the IEEE/CVF Conference on Computer Vision and
  Pattern Recognition}, pages 374--382, 2019.

\bibitem{you2017large}
Yang You, Igor Gitman, and Boris Ginsburg.
\newblock Large batch training of convolutional networks.
\newblock {\em arXiv preprint arXiv:1708.03888}, 2017.

\bibitem{pmlr-v139-zbontar21a}
Jure Zbontar, Li~Jing, Ishan Misra, Yann LeCun, and Stephane Deny.
\newblock Barlow twins: Self-supervised learning via redundancy reduction.
\newblock In Marina Meila and Tong Zhang, editors, {\em Proceedings of the 38th
  International Conference on Machine Learning}, volume 139 of {\em Proceedings
  of Machine Learning Research}, pages 12310--12320. PMLR, 18--24 Jul 2021.

\bibitem{Zhai2021ScalingVT}
Xiaohua Zhai, Alexander Kolesnikov, Neil Houlsby, and Lucas Beyer.
\newblock Scaling vision transformers.
\newblock {\em ArXiv}, abs/2106.04560, 2021.

\bibitem{zhang2021contrastive}
Yuhao Zhang, Hang Jiang, Yasuhide Miura, Christopher~D Manning, and Curtis
  Langlotz.
\newblock Contrastive learning of medical visual representations from paired
  images and text, 2021.

\bibitem{zhu2020eqco}
Benjin Zhu, Junqiang Huang, Zeming Li, Xiangyu Zhang, and Jian Sun.
\newblock Eqco: Equivalent rules for self-supervised contrastive learning.
\newblock {\em arXiv preprint arXiv:2010.01929}, 2020.

\bibitem{eqco}
Benjin Zhu, Junqiang Huang, Zeming Li, Xiangyu Zhang, and Jian Sun.
\newblock Eqco: Equivalent rules for self-supervised contrastive learning.
\newblock {\em arXiv preprint arXiv:2010.01929}, 2020.

\bibitem{zhu2020deformable}
Xizhou Zhu, Weijie Su, Lewei Lu, Bin Li, Xiaogang Wang, and Jifeng Dai.
\newblock Deformable detr: Deformable transformers for end-to-end object
  detection.
\newblock {\em arXiv preprint arXiv:2010.04159}, 2020.

\end{thebibliography}

\onecolumn
\appendix

\section{Experiment Details }

\begin{figure*}[h]
\centering
\includegraphics[scale=0.26]{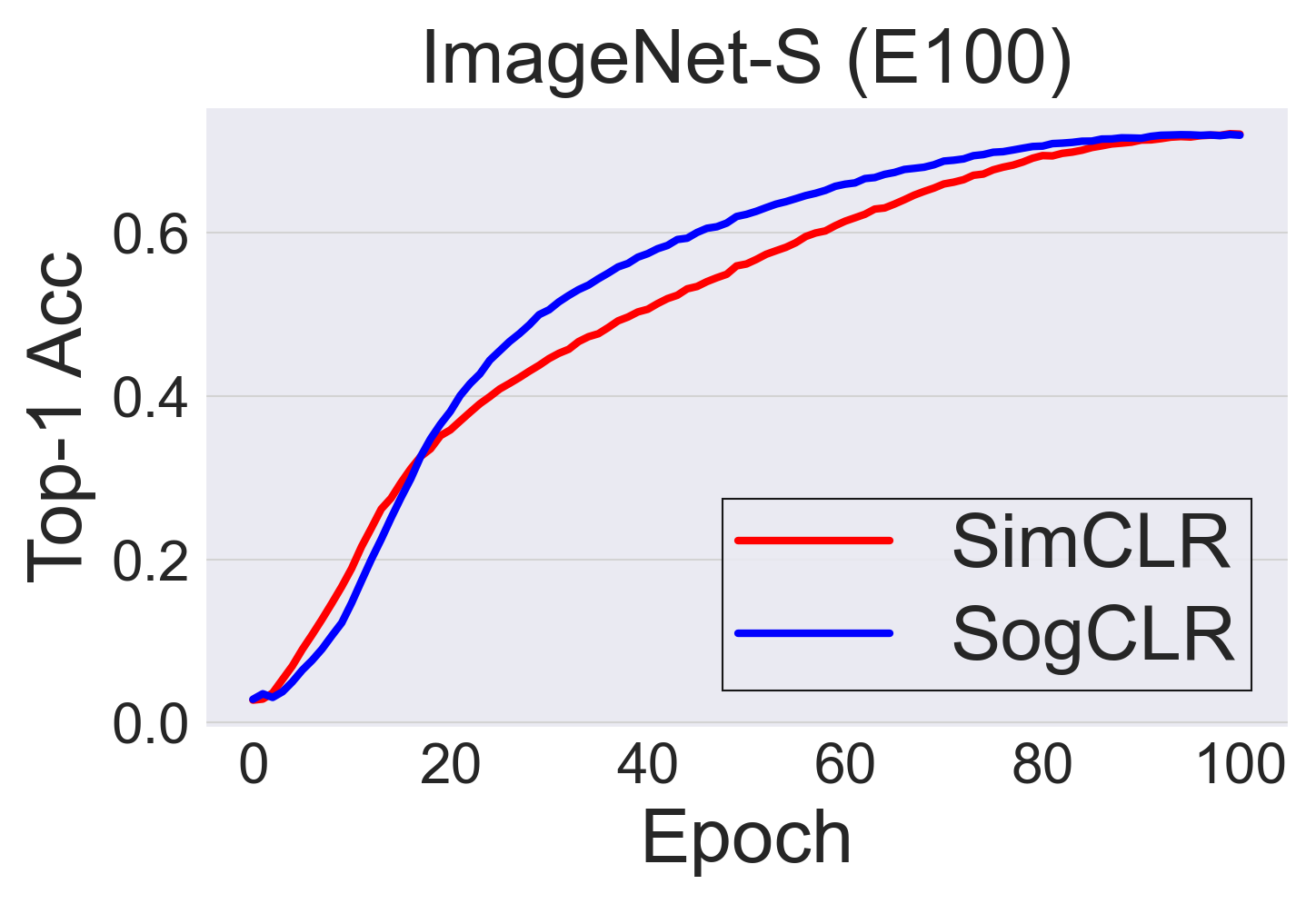}
\includegraphics[scale=0.26]{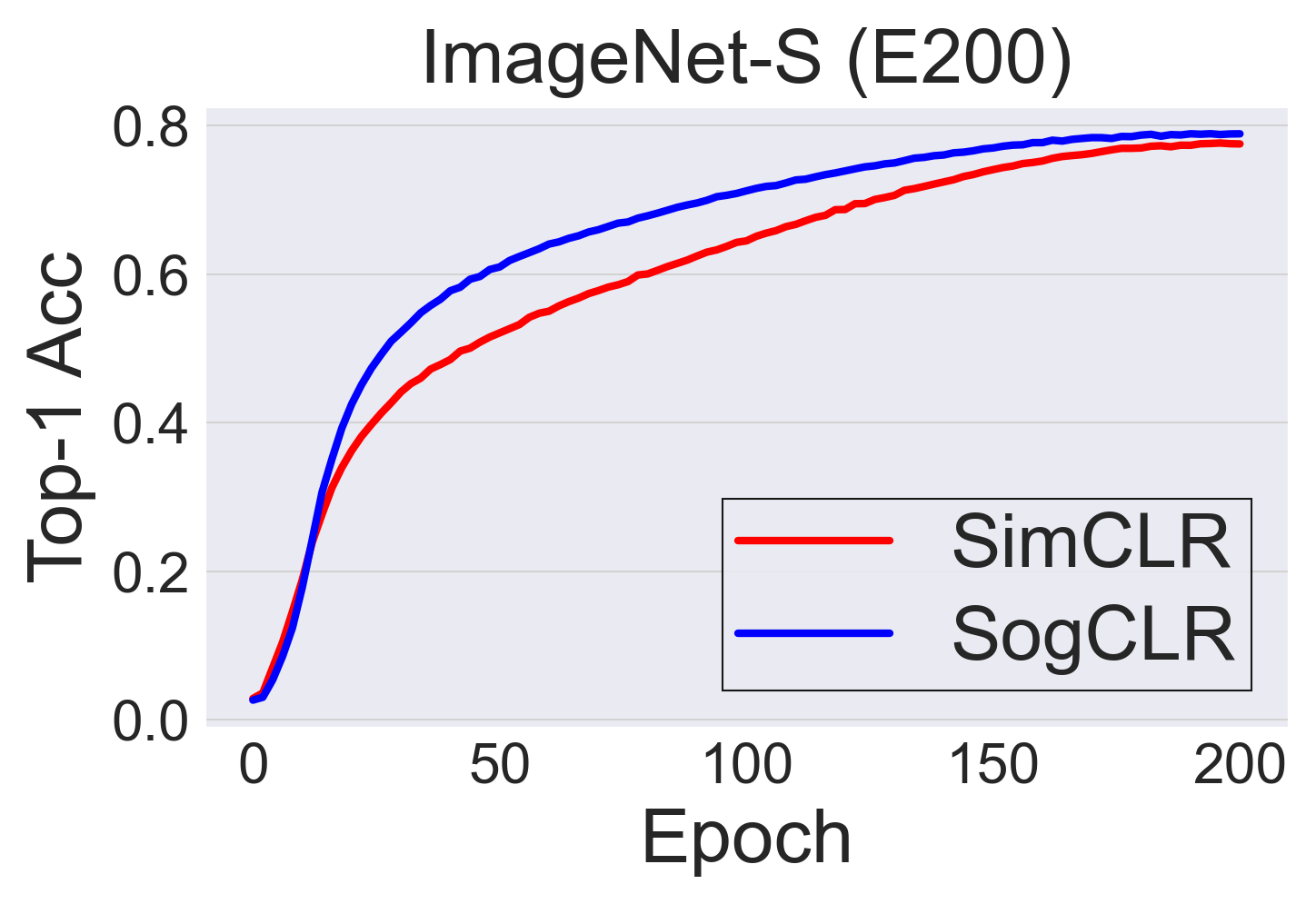}
\includegraphics[scale=0.26]{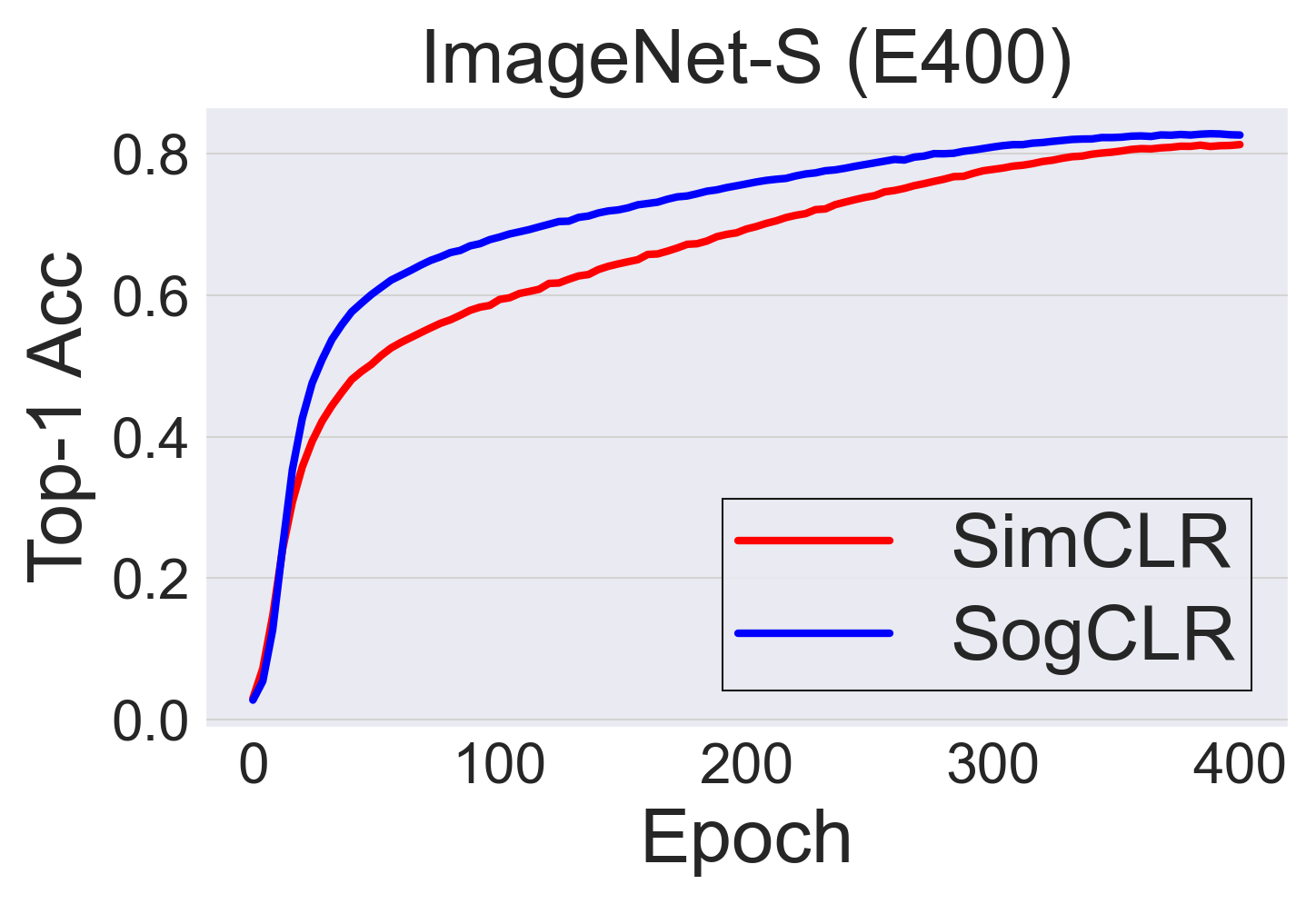}
\includegraphics[scale=0.26]{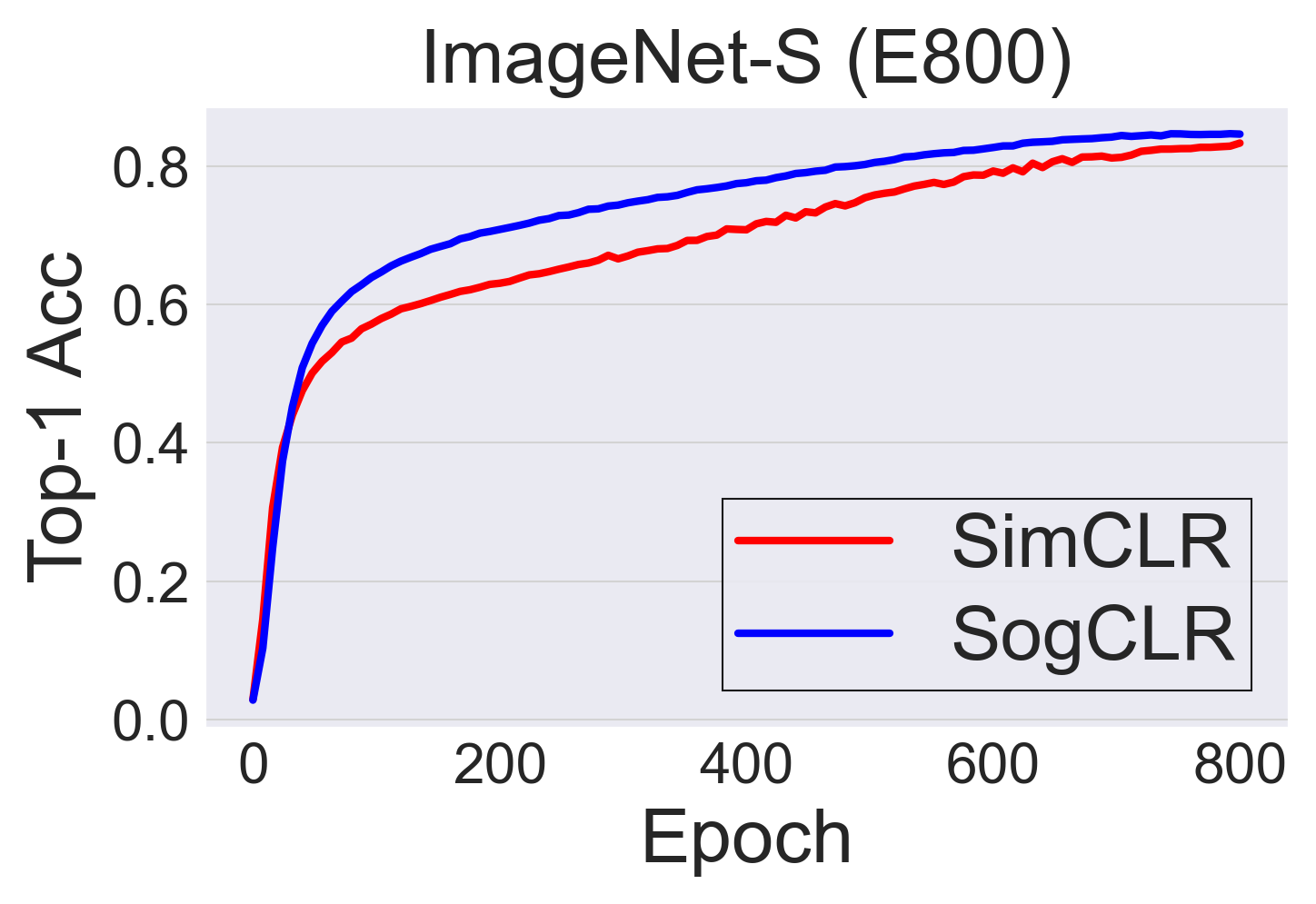}

\includegraphics[scale=0.26]{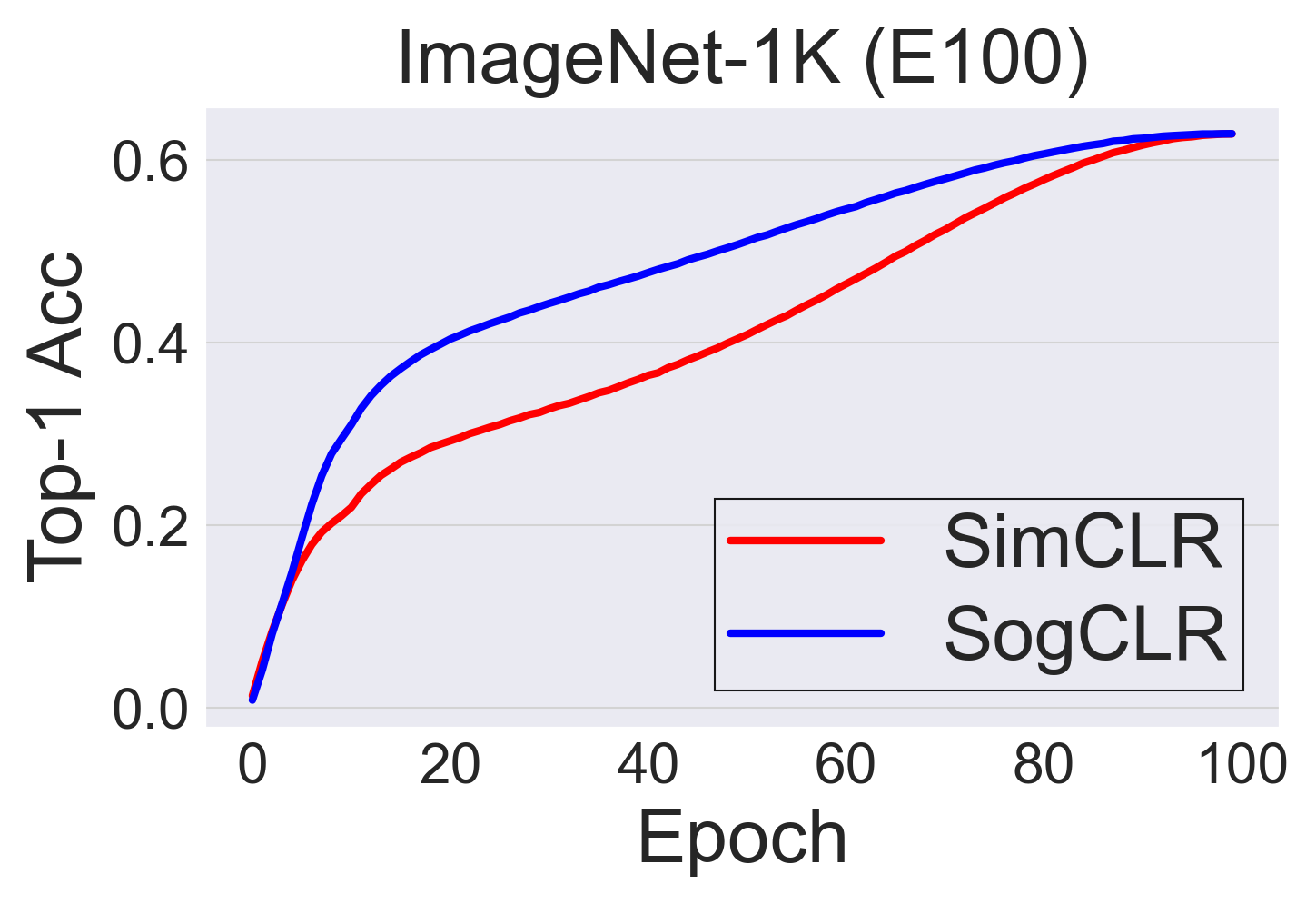}
\includegraphics[scale=0.26]{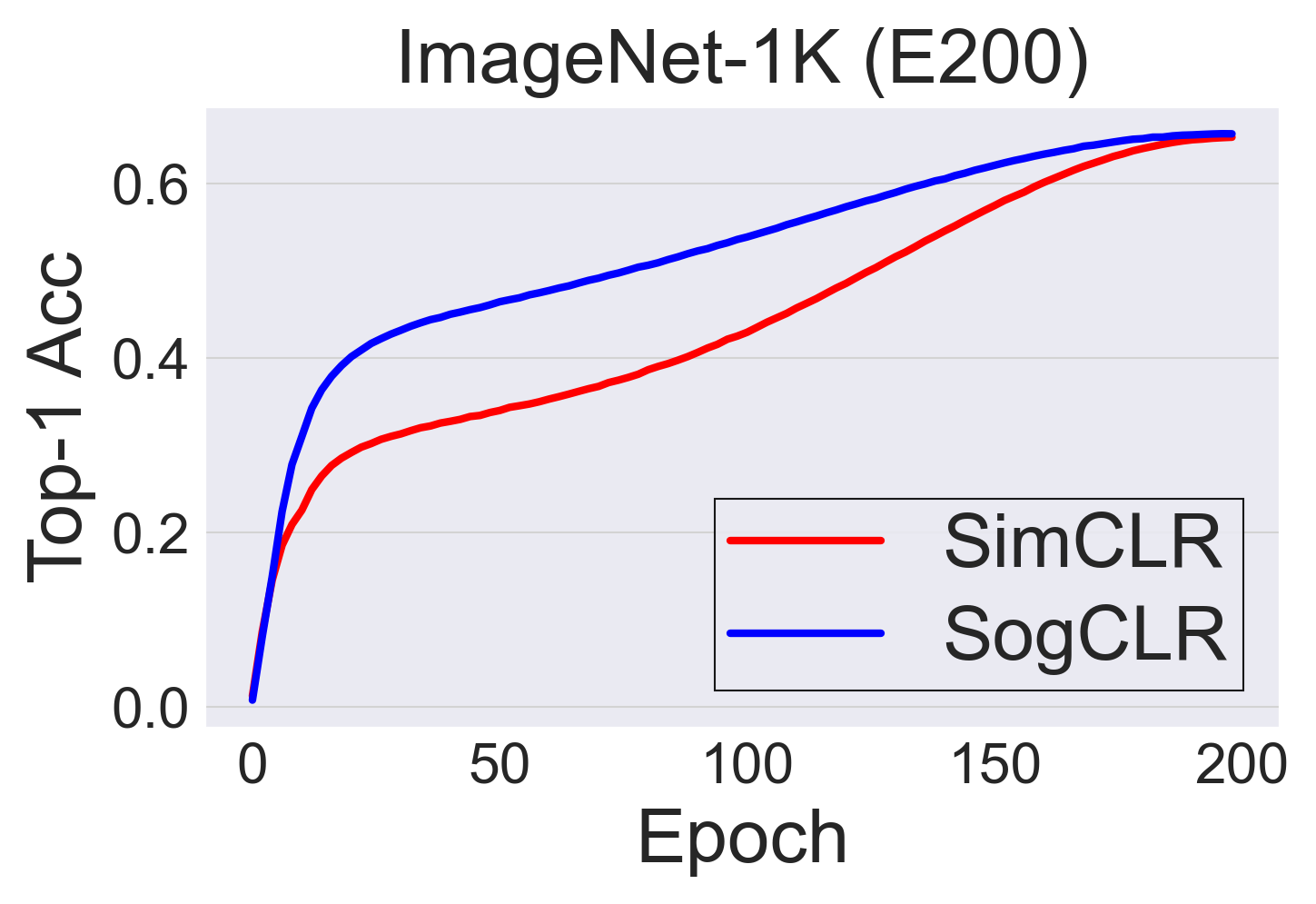}
\includegraphics[scale=0.26]{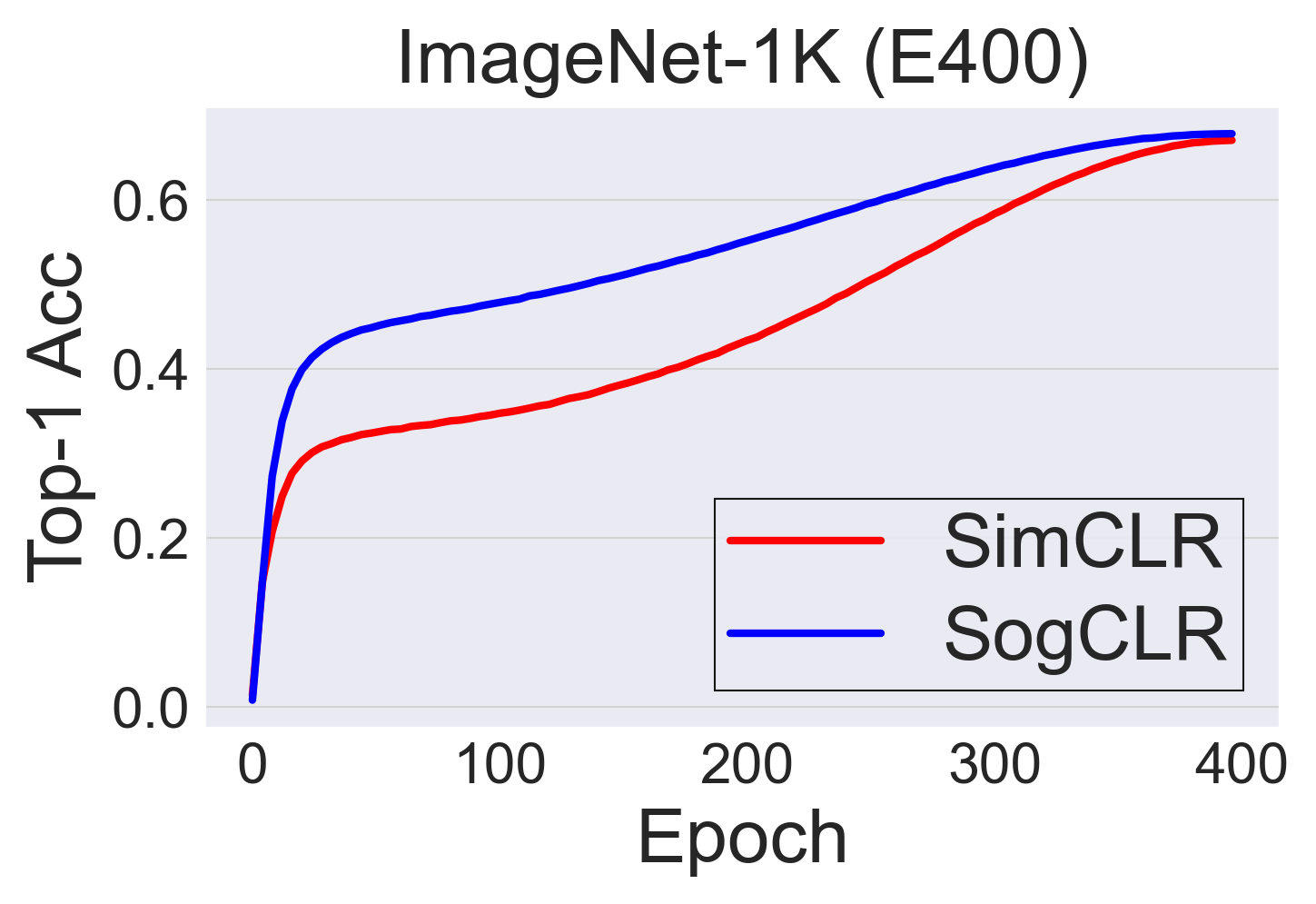}
\includegraphics[scale=0.26]{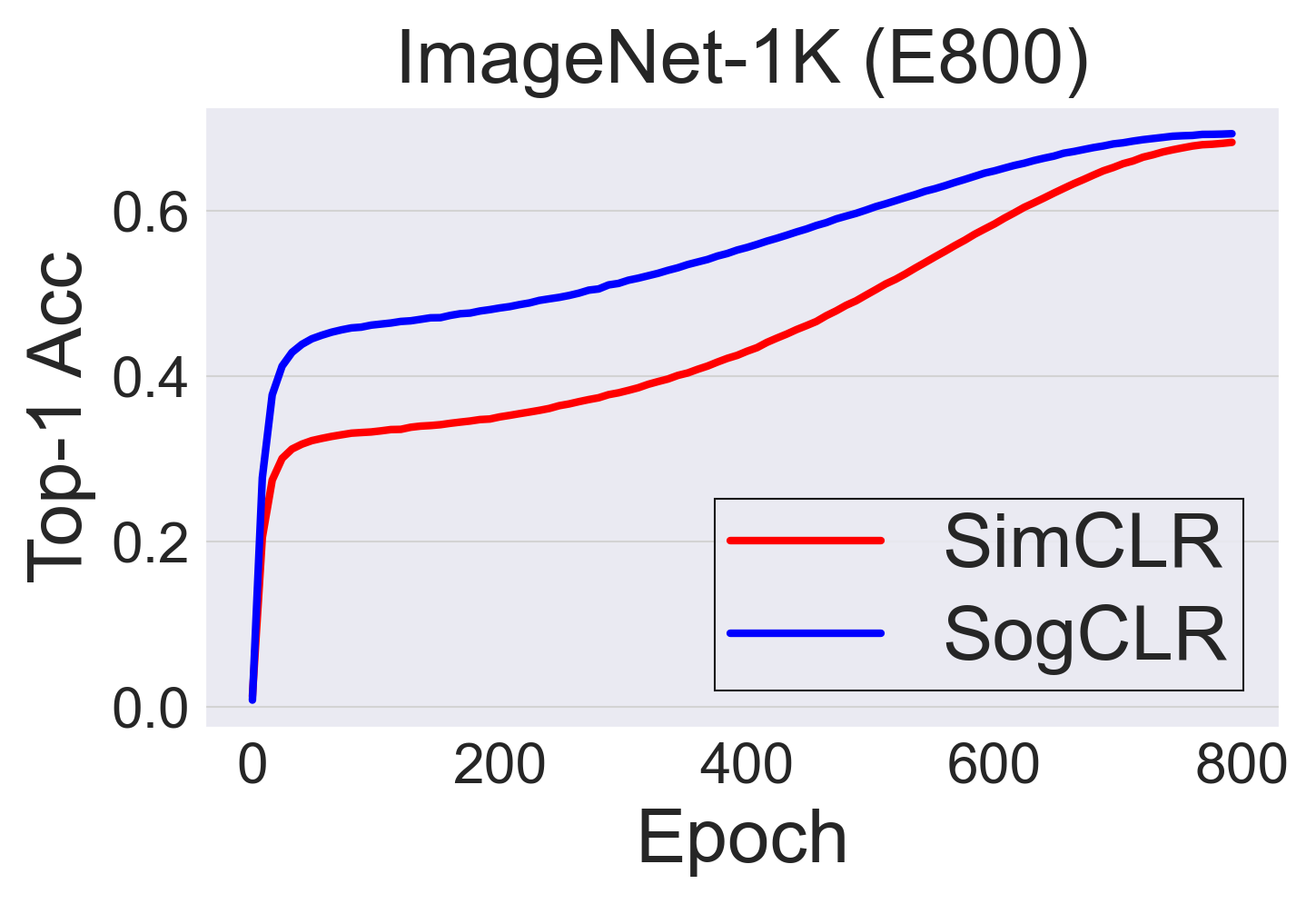}

\caption{Learning curve for top-1 accuracy by linear evaluation on ImageNet-S and ImageNet-1K training set trained on ResNet-50 using batch size of 256. } 
\label{fig:learning_curve_imagenet_100_1000_train} 
\end{figure*} 

\begin{table}[h]
\caption{CLIP-S hyper-parameters.}
\centering
\label{tab:clip_s_configs}
\scalebox{0.95}{
\begin{tabular}{|l|c|}
\hline
\multicolumn{1}{|c|}{Hyperparameter} & Value \\ \hline
embed\_dim & 512 \\ 
image\_resolution & 224$\times$224 \\ 
vision\_layers & {[}3,4,6,3{]} \\ 
vision\_width & 32 \\ 
vision\_patch\_size & null \\ 
context\_length & 77 \\ 
vocab\_size & 49408 \\ 
transformer\_width & 128 \\ 
transformer\_heads & 8 \\ 
transformer\_layers & 8 \\ \hline
\end{tabular}}
\end{table}

\begin{table}[h]
\centering
\caption{Top-1 linear evaluation accuracy trained on ResNet-50 under different number of epochs using batch size of 256 on ImageNet-1K for $\gamma = 1$ v.s. $\gamma < 1$ in Algorithm \ref{alg:sigclr}. }
\label{tab:4}
\scalebox{0.95}{
\begin{tabular}{ccccc}
\hline
$\gamma$\textbackslash{}Epoch & 100 & 200 & 400 & 800 \\ \hline
1.0 & 62.8 & 64.3 & 65.7 & 66.5\\ \hline
0.99 & 64.9 & 67.1 & 68.3 & 69.2 \\ 
0.9 & 65.0 & 66.9 & 68.1 & 69.2 \\ 
0.8 & \textbf{65.2} & 67.1 & 68.4 & 69.3 \\ 
0.7 & 65.0 & \textbf{67.1} & \textbf{68.7} & \textbf{69.4}\\ 
0.6 & 64.4 & 66.7 & 68.3 & 69.2 \\ \hline
\end{tabular}}
\end{table}

\begin{table}[H]
\centering
\caption{Linear evaluations with different nonlinear heads and multi-crop augmentation.}
\scalebox{0.95}{
\begin{tabular}{ccc}
\hline
Num of views & 3-layer proj. head & 4-layer proj. head \\ \hline
$2\times 224$ & 70.7 & 71.3 \\ 
$4\times 160 + 2 \times 96$& 71.7 & \textbf{72.5} \\ \hline
\end{tabular}}\label{tab:tricks}
\end{table}

\section{Notations in the Proofs}
In the following proofs, we abuse the notation: $g_i(\w; \A, \S_i)=g(\w; \x_i, \A, \S_i)=\frac{1}{|\S_i|}g(\w; \x_i, \A, \S_i)$ and $g_i(\w;\A,\B_i)=g(\w; \x_i, \A, \B_i)=\frac{1}{|\B_i|}g(\w; \x_i, \A, \B_i)$. In the following analysis, we assume $\x_i\in\B$ is independently sampled with replacement and $\A, \A'$ are also independently sampled for each sampled data independently though we abuse the same notations $\A, \A'$ for different data.   It is notable that $\E_{\B_i|\x_i}[g(\w; \x_i, \A, \B_i)] = g(\w; \x_i, \A, \S_i)$. 

We write the objective function as 
\begin{align*}
    F(\w) = F_1(\w) + F_2(\w)
\end{align*}
where we ignore the constant and 
\begin{align*}
    F_1(\w) &= -\E_{\x_i, \A,\A'} [E(\A(\x_i))^{\top}E(\A'(\x_i))]\\
    F_2(\w) &=\frac{\tau}{n}\sum_{\x_i\in\D}\E_{\A}\ln \left(\varepsilon_0 + g(\w; \x_i, \A, \S_i)\right) = \frac{\tau}{n}\sum_{\x_i\in\D}\E_{\A}f(g(\w; \x_i, \A, \S_i))
\end{align*}
where $f(g) = \ln(\varepsilon_0 + g)$.

\section{Proof of Theorem~\ref{thm:1}}
The SimCLR with the update~(\ref{eqn:SimCLR}) uses the following gradient estimator: 
\begin{align*}
    \v_t = \nabla F_1(\w_t; \B) + \frac{1}{B}\sum_{\x_i\in\B}\nabla g(\w; \x_i, \A, \B_i)\nabla f(g(\w; \x_i, \A, \B_i))
\end{align*}
 We make the following standard assumptions. 
\begin{ass}
We assume that there exist $\sigma, C_g, C_f, L_f, L_F$ such that
\begin{itemize}
    \item $\E[\|\nabla F_1(\w; \B) - \nabla F_1(\w)\|^2]\leq \frac{\sigma^2}{B}$
    \item $\E_{\B_i|\x_i}[\|g(\w; \x_i, \A, \B_i) - g(\w; \x_i, \A, \S_i)\|^2]\leq \frac{\sigma^2}{B}$ and   $\E_{\B_i|\x_i}[\|\nabla g(\w; \x_i, \A, \B_i) - \nabla g(\w; \x_i, \A, \S_i)\|^2]\leq \frac{\sigma^2}{B}$

    \item $\|\nabla g_i(\w; \A, \S_i)\|\leq C_g$
    \item $\|\nabla f(g)\|\leq C_f$, and $\nabla f(\cdot)$ is $L_f$ Lipschitz continuous
    \item $F$ is $L_F$-smooth. 
\end{itemize}
\end{ass}
It is notable that the above assumptions are mild or standard for convergence analysis. 

Below, we use $\E_t$ to denote the expectation over randomness at $t$-th iteration given history. First, we have
\begin{align*}
    &\E_t[F(\w_{t+1})]\leq \E_t[F(\w_t) + (\w_{t+1} - \w_t)^{\top}\nabla F(\w_t) + \frac{\eta^2 L_F}{2}\|\v_{t}\|^2]\\
    &= F(\w_t) - \eta\E_t [(\nabla F_1(\w_t; \B) + \frac{1}{B}\sum_{\x_i\in\B}\nabla g(\w; \x_i, \A, \B_i)\nabla f(g(\w; \x_i, \A, \B_i)))]^{\top}\nabla F(\w_t)\\
    &+ \frac{\eta^2 L_F}{2}\|\v_{t}\|^2]\\
    &= F(\w_t) - \eta\E_t [(\nabla F_1(\w_t; \B) + \frac{1}{B}\sum_{\x_i\in\B}\nabla g(\w; \x_i, \A, \B_i)\nabla f(g(\w; \x_i, \A, \S_i)))]^{\top}\nabla F(\w_t)]\\
   & + \eta\E_t [(\nabla F_1(\w_t; \B) + \frac{1}{B}\sum_{\x_i\in\B}\nabla g(\w; \x_i, \A, \B_i)\nabla f(g(\w; \x_i, \A, \S_i)))\\
   &- (\nabla F_1(\w_t; \B) + \frac{1}{B}\sum_{\x_i\in\B}\nabla g(\w; \x_i, \A, \B_i)\nabla f(g(\w; \x_i, \A, \B_i)))^{\top}\nabla F(\w_t)\\
    &+ \frac{\eta^2 L_F}{2}\|\v_{t}\|^2]\\
     &= F(\w_t) - \eta\|\nabla F(\w_t)\|^2 + \eta\E_t[\frac{1}{B}\sum_{\x_i\in\B}\|\nabla F(\w_t)\|C_gL_f\|g(\w; \x_i, \A, \B_i) - g(\w; \x_i, \A, \S_i)\|+ \frac{\eta^2 L_F}{2}\|\v_t\|^2]\\
    &= F(\w_t) - \eta\|\nabla F(\w_t)\|^2 + \frac{\eta}{2}\|\nabla F(\w_t)\|^2 +\frac{\eta C^2_gL^2_f}{2}\E_t[\frac{1}{B}\sum_{\x_i\in\B}\|g(\w; \x_i, \A, \B_i) - g(\w; \x_i, \A, \S_i)\|^2]+ \frac{\eta^2 L_F}{2}\E_t[\|\v_{t}\|^2]\\
     &= F(\w_t) - \frac{\eta}{2}\|\nabla F(\w_t)\|^2  +\frac{\eta C^2_gL^2_f\sigma^2}{2B}+ \frac{\eta^2 L_F}{2}\E_t[\|\v_{t}\|^2]
\end{align*}

Then we have
\begin{align*}
    &\E[\|\v_t - \nabla F(\w_t)\|^2] = \E[\|\nabla F_1(\w_t; \B)  - \nabla F_1(\w_t)\\
     &+ \frac{1}{B}\sum_{\x_i\in\B}\nabla g(\w; \x_i, \A, \B_i)\nabla f(g(\w; \x_i, \A, \B_i)) -  \frac{1}{n}\sum_{\x_i\in\D}\E_{\A}\nabla g(\w; \x_i, \A, \S_i)\nabla f(g(\w; \x_i, \A, \S_i))\|^2]\\
     &\leq 2\E[\|\nabla F_1(\w_t; \B)  - \nabla F_1(\w_t)\|^2] \\
     &+ 2\E[\|\frac{1}{B}\sum_{\x_i\in\B}\nabla g(\w; \x_i, \A, \B_i)\nabla f(g(\w; \x_i, \A, \B_i)) -  \frac{1}{n}\sum_{\x_i\in\D}\E_{\A}\nabla g(\w; \x_i, \A, \S_i)\nabla f(g(\w; \x_i, \A, \S_i))\|^2]\\
     &\leq \frac{2\sigma^2}{B} + 2\E[\|\frac{1}{B}\sum_{\x_i\in\B}\nabla g(\w; \x_i, \A, \B_i)\nabla f(g(\w; \x_i, \A, \B_i)) -  \frac{1}{n}\sum_{\x_i\in\D}\E_{\A}\nabla g(\w; \x_i, \A, \S_i)\nabla f(g(\w; \x_i, \A, \S_i))\|^2]
\end{align*}
To bound the second term, we have
\begin{align*}
    &\E[\|\frac{1}{B}\sum_{\x_i\in\B}\nabla g(\w; \x_i, \A, \B_i)\nabla f(g(\w; \x_i, \A, \B_i)) -  \frac{1}{n}\sum_{\x_i\in\D}\E_{\A}\nabla g(\w; \x_i, \A, \S_i)\nabla f(g(\w; \x_i, \A, \S_i))\|^2]\\
    &=\E[\|\frac{1}{B}\sum_{\x_i\in\B}\nabla g(\w; \x_i, \A, \B_i)\nabla f(g(\w; \x_i, \A, \B_i)) -  \frac{1}{B}\sum_{\x_i\in\B}\nabla g(\w; \x_i, \A, \B_i)\nabla f(g(\w; \x_i, \A, \S_i))\\
    &+ \frac{1}{B}\sum_{\x_i\in\B}\nabla g(\w; \x_i, \A, \B_i)\nabla f(g(\w; \x_i, \A, \S_i)) -    \frac{1}{n}\sum_{\x_i\in\D}\E_{\A}\nabla g(\w; \x_i, \A, \S_i)\nabla f(g(\w; \x_i, \A, \S_i))\|^2]\\
    & \leq \E[\frac{2}{B^2}B\sum_{\x_i\in\B}C_g^2 L_f^2\|g(\w; \x_i, \A, \B_i) - g(\w; \x_i, \A, \S_i)\|^2]\\
    &+ 2\E[\|\frac{1}{B}\sum_{\x_i\in\B}\nabla g(\w; \x_i, \A, \B_i)\nabla f(g(\w; \x_i, \A, \S_i)) -    \frac{1}{n}\sum_{\x_i\in\D}\E_{\A}\nabla g(\w; \x_i, \A, \S_i)\nabla f(g(\w; \x_i, \A, \S_i))\|^2]\\
    & = \frac{2C_g^2L_f^2\sigma^2}{B} \\
    &+ 2\E[\|\frac{1}{B}\sum_{\x_i\in\B}\nabla g(\w; \x_i, \A, \B_i)\nabla f(g(\w; \x_i, \A, \S_i)) - \frac{1}{B}\sum_{\x_i\in\B}[\nabla g(\w; \x_i, \A, \S_i)\nabla f(g(\w; \x_i, \A, \S_i))  \\
    &+\frac{1}{B}\sum_{\x_i\in\B}[\nabla g(\w; \x_i, \A, \S_i)\nabla f(g(\w; \x_i, \A, \S_i)) - \frac{1}{n}\sum_{\x_i\in\D}\E_{\A}\nabla g(\w; \x_i, \A, \S_i)\nabla f(g(\w; \x_i, \A, \S_i))\|^2]\\
        & = \frac{2C_g^2L_f^2\sigma^2}{B} + \frac{4C_f^2\sigma^2}{B} \\
    &+4\E[\|\frac{1}{B}\sum_{\x_i\in\B}[\nabla g(\w; \x_i, \A, \S_i)\nabla f(g(\w; \x_i, \A, \S_i)) - \frac{1}{n}\sum_{\x_i\in\D}\E_{\A}\nabla g(\w; \x_i, \A, \S_i)\nabla f(g(\w; \x_i, \A, \S_i))\|^2]\\
    & \leq  \frac{2C_g^2L_f^2\sigma^2}{B} + \frac{4C_f^2\sigma^2}{B} \\
    &+4\E[\|\frac{1}{B}\sum_{\x_i\in\B}[\nabla g(\w; \x_i, \A, \S_i)\nabla f(g(\w; \x_i, \A, \S_i))\|^2]\\
     & \leq  \frac{2C_g^2L_f^2\sigma^2}{B} + \frac{4C_f^2\sigma^2}{B} +4C_g^2C_f^2\\
\end{align*}
As a result, 
\begin{align*}
    \E[\|\v_t\|^2]\leq 2\|\nabla F(\w_t)\|^2 + \frac{C}{B} + 16C_g^2C_f^2,
\end{align*}
where $C$ is a proper constant. 
By combining the above results together, we have
\begin{align*}
  \E[F(\w_{t+1})]\leq   F(\w_t) - \frac{\eta}{2}\|\nabla F(\w_t)\|^2  +\frac{\eta C}{B}+ \eta^2 L_F\|\nabla F(\w_t)\|^2 + 16\eta^2L_FC_g^2C_f^2.
\end{align*} 
Then  with $\eta L_F\leq 1/4$, we have
\begin{align*}
    \E[\frac{1}{T}\sum_{t=1}^T\|\nabla F(\w_t)\|^2]\leq\frac{ 4(F(\w_1) - F_*) }{\eta T}+ 64\eta L_fC_f^2C_g^2 + \frac{4C}{B}, 
\end{align*}
which finises the proof. 

We can also sharpen the bound of $\E[\|\v_t - \nabla F(\w_t)\|^2]$ by noting that 

\begin{align*}
    &\E[\|\frac{1}{B}\sum_{\x_i\in\B}[\nabla g(\w; \x_i, \A, \S_i)\nabla f(g(\w; \x_i, \A, \S_i)) - \frac{1}{n}\sum_{\x_i\in\D}\E_{\A}\nabla g(\w; \x_i, \A, \S_i)\nabla f(g(\w; \x_i, \A, \S_i))\|^2]
    \\ &=\E[\|\frac{1}{B}\sum_{\x_i\in\B}[\nabla g(\w; \x_i, \A, \S_i)\nabla f(g(\w; \x_i, \A, \S_i)) -
    \frac{1}{B}\sum_{\x_i\in\B}\E_{\A}[\nabla g(\w; \x_i, \A, \S_i)\nabla f(g(\w; \x_i, \A, \S_i))+\\
    &\frac{1}{B}\sum_{\x_i\in\B}\E_{\A}[\nabla g(\w; \x_i, \A, \S_i)\nabla f(g(\w; \x_i, \A, \S_i))-\frac{1}{n}\sum_{\x_i\in\D}\E_{\A}\nabla g(\w; \x_i, \A, \S_i)\nabla f(g(\w; \x_i, \A, \S_i))\|^2]\\
    &\leq \frac{4C_g^2C_f^2}{B}.
\end{align*}

As a result, $ \E[\|\v_t - \nabla F(\w_t)\|^2]\leq  \frac{2C_g^2L_f^2\sigma^2}{B} + \frac{4C_f^2\sigma^2}{B} +\frac{4C_g^2C_f^2}{B}$, then  with $\eta L_F\leq 1/4$, we have
\begin{align*}
    \E[\frac{1}{T}\sum_{t=1}^T\|\nabla F(\w_t)\|^2]\leq\frac{ 4(F(\w_1) - F_*) }{\eta T}+ \frac{C}{B},
\end{align*}
which still has a dependence of $1/B$. However, we can set $\eta=O(1)$ and $T=O(1/\epsilon^2), B=O(1/\epsilon^2)$ in order to achieve an $\epsilon$-stationary solution.  

\section{Proof of Theorem~\ref{thm:2}}
First, we note that the gradient estimator $\m_t$ is 
\begin{align*}
\m_{t} &= \nabla F_1(\w_t; \B) + \frac{1}{B}\sum_{\x_i\in\B}\nabla f(u_{i,t})\underbrace{\frac{1}{2}(\nabla g(\w_t; \x_i, \A, \B_i) + \nabla g(\w_t; \x_i, \A', \B_i))}\limits_{\nabla g_i(\w_t; \A, \A', \B_i)}\\
\u_{i, t+1} & = (1-\gamma) \u_{i, t} + \gamma \underbrace{\frac{1}{2}(g(\w_t; \x_i, \A,  \B_i)+g(\w_t; \x_i, \A',  \B_i))}\limits_{g_i(\w_t; \A, \A', \B_i)}
\end{align*}
Define $g_i(\w) = \E_{\A}[g(\w; \x_i, \A, \S_i)]$. We can see that $\E_{\B_i, \A, \A'|\x_i}[g_i(\w_t; \A, \A', \B_i)] = g_i(\w_t)$, and $\E_{\B_i, \A, \A'|\x_i}[\nabla g_i(\w_t; \A, \A', \B_i)]=\nabla g_i(\w_t)$. 

We make the following assumptions. 
\begin{ass}
We assume that there exist $\sigma, C_g, C_f, L_f, L_F$ such that
\begin{itemize}
    \item $\E[\|\nabla F_1(\w; \B) - \nabla F_1(\w)\|^2]\leq \frac{\sigma^2}{B}$
    \item $\E_{\B_i|\x_i}[\|g(\w; \x_i, \A, \B_i) - g(\w; \x_i, \A, \S_i)\|^2]\leq \frac{\sigma^2}{B}$ and   $\E_{\B_i|\x_i}[\|\nabla g(\w; \x_i, \A, \B_i) - \nabla g(\w; \x_i, \A, \S_i)\|^2]\leq \frac{\sigma^2}{B}$
    \item $\|\nabla f(g)\|\leq C_f, \|\nabla g_i(\w)\|\leq C_g$, $\|\nabla F_1(\w)\|\leq C_{F_1}$, $\|\nabla g_i(\w; \A, \S_i)\|\leq C_g$
    \item $\nabla f(\cdot)$ is $L_f$ Lipschitz continuous
    \item $F$ is $L_F$-smooth. 
    \item $\|E(\z)\|\leq 1$, $\forall \z$
    \item $\E_{\A, \A'}\E_{\z\sim\S_i}|E(\A(\x_i))^{\top}E(\z)] - E(\A'(\x_i))^{\top}E(\z)|^2\leq \epsilon^2$ for any $\x_i\in\D$
\end{itemize}
\end{ass}
We note that under the above assumption we have

\begin{align*}
 &\E_{\A, \A'}\|g_i(\w_{t}; \A', \S_i) - g_i(\w_{t}; \A, \S_i)\|^2\\    &=\E_{\A, \A'}\|\E_{\z\sim\S_i}(\exp(E(\A(\x_i))^{\top}E(\z)/\tau) - \E_{\z\sim\S_i}(\exp(E(\A'(\x_i))^{\top}E(\z)/\tau)\|^2\\
 &\leq \E_{\A, \A'}\E_{\z\sim\S_i}\|(\exp(E(\A(\x_i))^{\top}E(\z)/\tau) - (\exp(E(\A'(\x_i))^{\top}E(\z)/\tau)\|^2\\
 &\leq C\E_{\A, \A'}\E_{\z\sim\S_i}\|E(\A(\x_i))^{\top}E(\z)/\tau - E(\A'(\x_i))^{\top}E(\z)/\tau\|^2\\
& \leq O(\epsilon^2)
\end{align*}
where $C$ is a proper constant that bounds the Lipschitz of $\exp(E(\dot)E(\cdot)/\tau)$. 

 We need the following lemma, whose proof can be found in~\cite{guo2021stochastic} and thus is omitted here.  
\begin{lemma}\label{lem:nonconvex_starter}
Consider a sequence $\w_{t+1} = \w_t - \eta \v_t$ and the $L_F$-smooth function $F$ and the step size $\eta_t L_F\leq 1/2$. 
\begin{align}\label{eq:nonconvex_starter}
	F(\w_{t+1}) & \leq F(\w_t) + \frac{\eta}{2}\Norm{\Delta_t}^2 - \frac{\eta}{2}\Norm{\nabla F(\w_t)}^2 - \frac{\eta}{4}\Norm{\v_t}^2,
\end{align}	
where $\Delta_t =\v_t - \nabla F(\w_t)$.
\end{lemma}

\begin{lemma}\label{lem:grad_recursion}
	Assume $\E_{\A, \A'}\|g_i(\w_{t}; \A', \S_i) - g_i(\w_{t}; \A, \S_i)\|^2\leq \epsilon^2$, we have
\begin{align*}
	& \E[\Norm{\Delta_{t}}^2] \leq(1-\beta) \E\left[\Norm{\Delta_{t-1}}^2\right] + \frac{2L_F^2\eta^2\E\left[\Norm{\v_{t-1}}^2\right]}{\beta} + \frac{14\beta L_f^2C_g^2}{n}\E\left[\Norm{\Xi_{t}}^2\right]  + \frac{14\beta L_f^2C_g^2}{n}\E\left[\Norm{\u_{t} -\u_{t+1}}^2\right]\\
	&+ \frac{\beta^2C}{B} + 5\beta C_g^2L_f^2\epsilon^2.
\end{align*}
where $\Delta_t = \v_t - \nabla F(\w_t)$, $\Xi_t =  \u_{t+1} - \bg(\w_t)$, $\bg(\w)=(g_1(\w), \ldots, g_n(\w))$, and $C$ is a proper constant.
\end{lemma}	
\begin{proof}
	We define that $\Delta_t= \v_t  - \nabla F(\w_t)$. Below, for the analysis of $\Delta_{t}$, we note that $\u_{t}$ is independent of the randomness in $\B, \A, \A'$. 
	Based on the update rule $\v_{t} = (1-\beta)\v_{t-1} + \beta \m_{t}$, we have
	\begin{align*}
		& \Norm{\Delta_{t}}^2 = \Norm{\v_{t}  - \nabla F(\w_{t})}^2\\
		& = \Norm{(1-\beta)\v_{t-1} + \beta (\nabla F_1(\w_t; \B)+\frac{1}{B}\sum_{i\in\B}\nabla f([\u_{t-1}]_i)\nabla g_i (\w_{t};\A, \A', \B_i)) - \nabla F(\w_{t})}^2\\
		&= \left\|\underbrace{(1-\beta) (\v_{t-1} - \nabla F(\w_{t-1}))}_{A_1}+ \underbrace{(1-\beta)(\nabla F(\w_{t-1}) - \nabla F(\w_{t}))}_{A_{2}}\right.\\
		&+  \underbrace{\beta \left( \frac{1}{B}\sum_{i\in\B}\nabla f(g_i(\w_{t}))\nabla g_i (\w_{t};\A, \A', \B_i) -  \frac{1}{B}\sum_{i\in\B}\frac{1}{2}\sum_{\a=\A, \A'}\nabla f(g_i(\w_{t}; \a, \S_i))\nabla g_i (\w_{t};\a, \B_i)\right)}_{A_{3}} \\
		& \left. + \underbrace{\beta \left(\frac{1}{B}\sum_{i\in\B}\nabla f([\u_{t-1}]_i)\nabla g_i (\w_t;\A,\A',\B_i) - \frac{1}{B}\sum_{i\in\B}\nabla f(g_i(\w_{t}))\nabla g_i (\w_{t};\A, \A', \B_i)\right)}_{A_4}\right.\\
		& \left. +\underbrace{\beta \left(\nabla F_1(\w_t; \B) + \frac{1}{2B}\sum_{i\in\B}\left(\nabla f(g_i(\w_{t}; \A, \S_i))\nabla g_i (\w_{t};\A,\B_i)  + \nabla f(g_i(\w_{t}; \A', \S_i))\nabla g_i (\w_{t};\A',\B_i)\right)- \nabla F(\w_{t}) \right)}_{A_5}\right\|^2.
	\end{align*}
Note that $\E\left[\inner{A_1}{A_5}\right] = \E\left[\inner{A_2}{A_5}\right] = 0$. Then, 
\begin{align*}
&	\E_t\left[\Norm{A_1 + A_2 + A_3 + A_4+A_5}^2\right]  = \Norm{A_1}^2 + \Norm{A_2}^2 + \E_t\left[\Norm{A_3}^2\right] + \E_t\left[\Norm{A_4}^2\right] + \E_t\left[\Norm{A_5}^2\right]+ 2\inner{A_1}{A_2} \\
	&+ 2\E_t\left[\inner{A_1}{A_3}\right] + 2\E_t\left[\inner{A_1}{A_4}\right] + 2\E_t\left[\inner{A_2}{A_3}\right]+ 2\E_t\left[\inner{A_2}{A_4}\right]+2\E_t\left[\inner{A_3}{A_4}\right]+2\E_t\left[\inner{A_3}{A_5}\right] + 2\E_t\left[\inner{A_4}{A_5}\right].
\end{align*}
Based on Young's inequality for products, we have $2\inner{\a}{\b}\leq \frac{\Norm{\a}^2c}{2} + \frac{2\Norm{\b}^2}{c}$ for $c>0$. 
\begin{align*}
		& \E_t\left[\Norm{A_1 + A_2 + A_3 + A_4 + A_5}^2\right] \\
		&\leq (1+\beta)\Norm{A_1}^2 + (3+3/\beta)\Norm{A_2}^2 + (4+3/\beta)\E_t\left[\Norm{A_3}^2\right] +(4+\frac{3}{\beta})\E_t\left[\Norm{A_4}^2\right]+ 3\E_t\left[\Norm{A_5}^2\right].
\end{align*}
Thus, we have
\begin{align}\label{eq:grad_single_iter}
	\E_t[\Norm{\Delta_t}^2] &\leq (1-\beta)\Norm{\Delta_{t-1}}^2 + (3+3/\beta)\Norm{A_2}^2 + (4+3/\beta)\E_t[\Norm{A_3}^2] + (4+3/\beta)\E_t[\Norm{A_4}^2] + 3\E_t\left[\Norm{A_5}^2\right].
\end{align}
Moreover, we have
\begin{align}\label{eq:bound_A2}
	\Norm{A_2}^2 &= (1-\beta)^2\Norm{\nabla F(\w_{t-1}) - \nabla F(\w_{t})}^2 \leq  (1-\beta)^2\eta^2 L_F^2\Norm{\v_{t-1}}^2 ,\\\nonumber
	\E_t[\|A_3\|^2]&\leq \beta^2 C_g^2L_f^2\E_{\A}\|g_i(\w_{t}) - g_i(\w_{t}; \A, \S_i)\|^2\leq \beta^2 C_g^2L_f^2\E_{\A}\|\E_{\A'}g_i(\w_{t}; \A', \S_i) - g_i(\w_{t}; \A, \S_i)\|^2\\
	&\leq \beta^2 C_g^2L_f^2\E_{\A, \A'}\|g_i(\w_{t}; \A', \S_i) - g_i(\w_{t}; \A, \S_i)\|^2\\
	\E_t[\Norm{A_4}^2] &\leq \E_t[\frac{\beta^2}{B}\sum_{i\in\B}\Norm{\nabla g_i(\w_{t};\A, \A', \B_i)}^2\Norm{\nabla f([\u_{t-1}]_i) - \nabla f(g_i(\w_{t}))}^2]\\\nonumber
	& \leq \E_t[\frac{\beta^2L_f^2}{B}\sum_{i\in\B}  \Norm{\nabla g_i(\w_{t};\A, \A', \B_i)}^2\Norm{[\u_{t-1}]_i-g_i(\w_{t})}^2]\\
	&\leq \beta^2 L_f^2C_g^2 \E_t\left[\frac{1}{B}\sum_{i\in\B}\Norm{[\u_{t-1}]_i-g_i(\w_t)}^2\right]
\end{align}
Since the update rule of $\u_{t}$ is based on $\B$, we have
\begin{align*}
\E_t\left[\frac{1}{B}\sum_{i\in\B}\Norm{[\u_{t-1}]_i-g_i(\w_{t})}^2\right] = \frac{1}{n}\E_t \left[\Norm{\u_{t-1} - \bg(\w_{t})}^2\right],
\end{align*}
where  and $\bg(\w_t)\coloneqq \left[g_1(\w_t),\cdots, g_n(\w_t)\right]^\top$. Then,  $\Norm{\u_{t-1} - \bg(\w_t)}^2=\sum_{i\in\D}\Norm{[\u_{t-1}]_i - g_i(\w_t)}^2$.
Thus, we also have
\begin{align}\nonumber
	& \E_t\left[\Norm{A_5}^2\right] \leq \frac{2\beta^2\sigma^2}{B} \\
	&+  2 \beta^2\E_t\left[\Norm{\left( \frac{1}{2B}\sum_{i\in\B}\left(\nabla f(g_i(\w_{t}; \A, \S_i))\nabla g_i (\w_{t};\A,\B_i)  + \nabla f(g_i(\w_{t}; \A', \S_i))\nabla g_i (\w_{t};\A',\B_i)\right)- \nabla F_2(\w_{t}) \right)}^2\right]\\\nonumber
    &\leq \frac{2\beta^2\sigma^2}{B} +  4\beta^2 \E_t\left[\frac{1}{B}\sum_{i\in\B}\E_{\B_i}\Norm{\left(\nabla f(g_i(\w_{t}; \A, \S_i))\nabla g_i (\w_{t+1};\A, \B_i) - \E_{\B_i}\nabla f(g_i(\w_{t}; \A, \S_i))\nabla g_i (\w_{t}; \A, \B_i)\right)}^2\right]\\\nonumber
	& +  8\beta^2 \E_{\B}\left[\frac{1}{B}\sum_{i\in\B}\E_\A\Norm{\nabla f(g_i(\w_{t}; \A, \S_i))\nabla g_i (\w_{t};\A, \S_i) - \E_{\A}\nabla f(g_i(\w_{t}; \A, \S_i))\nabla g_i (\w_{t};\A, \S_i)}^2\right]\\\nonumber
	& +  8\beta^2 \E_{\B}\left[\Norm{\frac{1}{B}\sum_{i\in\B}\E_{\A}\nabla f(g_i(\w_{t}; \A, \S_i))\nabla g_i (\w_{t};\A, \S_i) - \frac{1}{n}\sum_{i\in\D}\E_{\A}\nabla f(g_i(\w_{t}; \A, \S_i))\nabla g_i (\w_{t};\A, \S_i)}^2\right]\\\nonumber
	&\leq \frac{\beta^2C}{B}\label{eq:bound_A4}  ,
\end{align}
where $C$ is some proper constant.

Define that $\Xi_t= \E\left[\frac{1}{n}\Norm{\u_{t} - \bg(\w_t)}^2\right]$. By combining the above inequalities we have
\begin{align*}
	& \E[\Norm{\Delta_{t}}^2] \leq(1-\beta) \E\left[\Norm{\Delta_{t-1}}^2\right] + \frac{3L_F^2\eta^2\E\left[\Norm{\v_{t-1}}^2\right]}{\beta} + \frac{14\beta L_f^2C_g^2}{n}\E\left[\Norm{\Xi_{t}}^2\right]  + 14\beta L_f^2C_g^2\E\left[\Norm{\u_{t} -\u_{t-1}}^2\right]\\
	&+ \frac{\beta^2C}{B} + 7\beta C_g^2L_f^2\epsilon^2.
\end{align*}
\end{proof}	

\begin{lemma}\label{lem:fval_recursion}
	If $\gamma\leq 1/5$, the following equation holds.
	\begin{align}
		\E\left[\Xi_{t+1}\right] &\leq \left(1-\frac{\gamma B}{4n}\right)\E\left[\Xi_{t}\right] + \frac{5n\eta^2C_g^2\E\left[\Norm{\v_{t}}^2\right]}{\gamma B} + \frac{2\gamma^2\sigma^2B}{nB} -\frac{1}{4n}\E\left[\Norm{\u_{t+1} - \u_t}^2\right].
	\end{align}	
\end{lemma}	

By Combining Lemma 1, 2, and Lemma 3, we can prove the final theorem.  
\begin{align}\label{eq:fisco_first_eq}
	& \E\left[F(\w_{t+1}) - F^*\right] \leq \E\left[F(\w_t)- F^*\right] + \frac{\eta}{2}\E\left[\Delta_t\right] - \frac{\eta}{2}\E\left[\Norm{\nabla F(\w_t)}^2\right] - \frac{\eta}{4}\E\left[\Norm{\v_t}^2\right]\\\label{eq:fisco_second_eq}
&\E\left[\Delta_{t+1}\right] \leq (1-\beta) \E\left[\Delta_t\right] + \frac{3L_F^2\eta^2}{\beta}\E\left[\Norm{\v_t}^2\right] + 14\beta L_f^2C_g^2\E\left[\Xi_{t+1}\right] + \frac{\beta^2C}{B}+ \frac{14\beta L_f^2 C_g^2}{n}\E\left[\Norm{\u_{t+1} - \u_t}^2\right]\nonumber\\
&+ 7\beta C_g^2L_f^2\epsilon^2\\\label{eq:fisco_third_eq}
&	\E\left[\Xi_{t+1}\right] \leq \left(1-\frac{\gamma B}{4n}\right)\E\left[\Xi_{t}\right] + \frac{5n\eta^2 C_g^2\E\left[\Norm{\v_t}^2\right]}{\gamma B} + \frac{2\gamma^2\sigma^2B}{nB} - \frac{1}{4n}\E\left[\Norm{\u_{t+1} - \u_t}^2\right].
\end{align} 	
Summing \eqref{eq:fisco_first_eq}, $\frac{\eta}{\beta}\times$\eqref{eq:fisco_second_eq}, and $\frac{56L_f^2C_g^2n\eta}{\gamma B}\times$\eqref{eq:fisco_third_eq} leads to
\begin{align*}
	& \E\left[(F(\w_{t+1})- F^*) + \frac{\eta}{\beta}\Delta_{t+1} + \frac{56L_f^2C_g^2n\eta}{\gamma B}\left(1-\frac{\gamma B}{4n}\right)\Xi_{t+1}\right]\\
	& \leq  \E\left[(F(\w_t)- F^*) + \frac{\eta}{\beta}\left(1-\frac{\beta }{2}\right)\Delta_t + \frac{56L_f^2C_g^2n\eta}{\gamma B}\left(1-\frac{\gamma B}{4n}\right)\Xi_t\right] - L_f^2C_g^2\eta\left(\frac{14n}{\gamma B}-14\right)\E\left[\frac{1}{n}\Norm{\u_{t+1} - \u_t}^2\right]\\
	& \quad\quad\quad -\frac{\eta}{2}\E\left[\Norm{\nabla F(\w_t)}^2\right] - \eta\left(\frac{1}{4} - \frac{3L_F^2\eta^2}{\beta^2} - \frac{280L_f^2 n^2C_g^4\eta^2 }{\gamma^2 B^2}\right)  \E\left[\Norm{\v_t}^2\right] + \frac{\beta\eta C}{B} + \frac{112\eta\gamma L_f^2 C_g^2\sigma^2}{B}+7\eta C_g^2L_f^2\epsilon^2.
\end{align*}
If $\gamma \leq \frac{n}{B}$, we have $\frac{14n}{\gamma B}-14\geq 0$. Set $\beta=O(\min(B\epsilon^2), \frac{2}{7})\}$, $\gamma = \min\left\{O(B\epsilon^2), \frac{5n}{14B}\right\}$, and $\eta = \min\left\{\frac{\beta}{6 L_F}, \frac{\gamma B}{50L_f n C_g^2}\right\}$. Define the Lyapunov function as $\Phi_t\coloneqq (F(\w_t) - F^*) + \frac{\eta}{\beta}\Delta_t + \frac{56L_f^2 C_g^2}{B}\frac{\eta}{\gamma}\left(1-\frac{\gamma B}{4n}\right)\Xi_t$. If we initialize $\v_0=0$, we have $\E[\Delta_1]\leq 2C_{F_1}+2C_f^2 C_g^2$.  Then,
\begin{align}\label{eq:fisco_grad_norm_bound}
	\frac{1}{T}\sum_{t=1}^T \E\left[\Norm{\nabla F(\w_t)}^2\right] \leq \frac{2\Lambda_\Phi^1}{\eta T} + \frac{2\beta C}{B} + \frac{224 \gamma L_f^2 C_g^2\sigma^2}{B} + 14 C_g^2L_f^2\epsilon^2,
\end{align}
where we define $\E\left[\Phi_1\right] \leq \Delta_F + \frac{1}{6L_F}(2C_{F_1}^2+2C_f^2 C_g^2) + \frac{C\Xi_1}{n}=: \Lambda_\Phi^1$. 
After $T=O(\max(\frac{n}{B^2\epsilon^4}, \frac{1}{B\epsilon^4}))$ iterations, we have $\frac{1}{T}\sum_{t=1}^T \E\left[\Norm{\nabla F(\w_t)}^2\right] \leq O(\epsilon^2)$.

\section{Proof of Corollary~\ref{thm:3}}
The proof is similar to that of Theorem~\ref{thm:2}, except that the bound $\Delta_t$, which is shown below. 

	\begin{align*}
		& \Norm{\Delta_{t}}^2 = \Norm{\v_{t}  - \nabla F(\w_{t})}^2\\
		& = \Norm{(1-\beta)\v_{t-1} + \beta (\nabla F_1(\w_t; \B)+\frac{1}{B}\sum_{i\in\B}\nabla f(g_i(\w_t; \A, \B_i))\nabla g_i (\w_{t};\A, \B_i)) - \nabla F(\w_{t})}^2\\
		&= \left\|\underbrace{(1-\beta) (\v_{t-1} - \nabla F(\w_{t-1}))}_{A_1}+ \underbrace{(1-\beta)(\nabla F(\w_{t-1}) - \nabla F(\w_{t}))}_{A_{2}}\right.\\
		& \left. + \underbrace{\beta \left(\frac{1}{B}\sum_{i\in\B}\nabla f(g_i(\w_t; \A, \B_i))\nabla g_i (\w_{t+1};\A,\B_i) - \frac{1}{B}\sum_{i\in\B}\nabla f(g_i(\w_{t}; \A, \S_i))\nabla g_i (\w_{t};\A, \B_i)\right)}_{A_3}\right.\\
		& \left. +\underbrace{\beta \left(\nabla F_1(\w_t; \B) + \frac{1}{B}\sum_{i\in\B}\nabla f(g_i(\w_{t}; \A, \S_i))\nabla g_i (\w_{t};\A,\B_i)  - \nabla F(\w_{t}) \right)}_{A_4}\right\|^2.
	\end{align*}
	
	Note that $\E\left[\inner{A_1}{A_4}\right] = \E\left[\inner{A_2}{A_4}\right] = 0$. Then, 
\begin{align*}
&	\E_t\left[\Norm{A_1 + A_2 + A_3 + A_4}^2\right]  = \Norm{A_1}^2 + \Norm{A_2}^2 + \E_t\left[\Norm{A_3}^2\right] + \E_t\left[\Norm{A_4}^2\right] + 2\inner{A_1}{A_2} \\
	&+ 2\E_t\left[\inner{A_1}{A_3}\right] + 2\E_t\left[\inner{A_1}{A_4}\right] + 2\E_t\left[\inner{A_2}{A_3}\right]+ 2\E_t\left[\inner{A_2}{A_4}\right]+2\E_t\left[\inner{A_3}{A_4}\right]
\end{align*}
Based on Young's inequality for products, we have $2\inner{\a}{\b}\leq \frac{\Norm{\a}^2c}{2} + \frac{2\Norm{\b}^2}{c}$ for $c>0$. 
\begin{align*}
		& \E_t\left[\Norm{A_1 + A_2 + A_3 + A_4 + A_5}^2\right] \\
		&\leq (1+\beta)\Norm{A_1}^2 +C/\beta\Norm{A_2}^2 + C/\beta\E_t\left[\Norm{A_3}^2\right] +C\E_t\left[\Norm{A_4}^2\right],
\end{align*}
where $C$is a proper constant. 
We can show that $\E[\|A_3\|^2]\leq \beta^2\frac{C}{B}$ and  $\E[\|A_4\|^2]\leq \beta^2\frac{C}{B}$ for some constant $C$. Then we have
\begin{align*}
	& \E[\Norm{\Delta_{t}}^2] \leq(1-\beta) \E\left[\Norm{\Delta_{t-1}}^2\right] + \frac{C\eta^2\E\left[\Norm{\v_{t-1}}^2\right]}{\beta} + \frac{\beta C}{B}  + \frac{\beta^2 C}{B} 
\end{align*}
Combining this inequality with lemma~\ref{lem:nonconvex_starter} and with $\eta\leq O(\beta)$, we can prove an optimization error 
\begin{align*}
    \E[\frac{1}{T}\sum_t\|\nabla F(\w_t)\|^2]\leq O(\frac{1}{\eta T} + \frac{1}{\beta T} + \frac{\beta}{B} + \frac{1}{\beta} )
\end{align*}

\section{Proof of Theorem~\ref{thm:4}}
The proof is similar to that of Theorem~\ref{thm:2}, except that the bound $\Delta_t$, which is shown below. 

\begin{align*}
		& \Norm{\Delta_{t}}^2 = \Norm{\v_{t}  - \nabla F(\w_{t})}^2\\
		& = \Norm{(1-\beta)\v_{t-1} + \beta (\nabla F_1(\w_t; \B)+\frac{1}{B}\sum_{i\in\B}\nabla f([\u_{t}]_i)\nabla g_i (\w_{t};\A, \A', \B_i)) - \nabla F(\w_{t})}^2\\
		&= \left\|\underbrace{(1-\beta) (\v_{t-1} - \nabla F(\w_{t-1}))}_{A_1}+ \underbrace{(1-\beta)(\nabla F(\w_{t-1}) - \nabla F(\w_{t}))}_{A_{2}}\right.\\
		& \left. + \underbrace{\beta \left(\frac{1}{B}\sum_{i\in\B}\nabla f([\u_{t}]_i)\nabla g_i (\w_{t+1};\A,\A',\B_i) - \frac{1}{B}\sum_{i\in\B}\nabla f(g_i(\w_{t}))\nabla g_i (\w_{t};\A, \A', \B_i)\right)}_{A_4}\right.\\
		& \left. +\underbrace{\beta \left(\nabla F_1(\w_t; \B) + \frac{1}{2B}\sum_{i\in\B}\left(\nabla f(g_i(\w_{t}))\nabla g_i (\w_{t};\A,\B_i)  + \nabla f(g_i(\w_{t}))\nabla g_i (\w_{t};\A',\B_i)\right)- \nabla F(\w_{t}) \right)}_{A_5}\right\|^2.
	\end{align*}
	from which we can see that $A_3$ is gone in the proof of Theorem~\ref{thm:2}, which is the source to cause the error depends on $\epsilon$. Then we can follow the same analysis to finish the proof, which is omitted here due to that it is almost a duplicate of Theorem~\ref{thm:2}. 

\end{document}